\documentclass{article}  
\usepackage{amsmath}
\usepackage{amssymb}
\usepackage{multirow}
\usepackage{soul}
\usepackage{tcolorbox}
\usepackage{wrapfig}
\tcbuselibrary{listings,breakable}
\usepackage[english]{babel}
\usepackage{natbib}
\usepackage{tikz}
\usetikzlibrary{bayesnet}
\usetikzlibrary{arrows}
\usepackage{color}

\usepackage[letterpaper,top=2cm,bottom=2cm,left=3.75cm,right=3.75cm,marginparwidth=1.75cm]{geometry}

\usepackage{amsmath}
\usepackage{graphicx}
\usepackage[colorlinks=true, allcolors=blue]{hyperref}

\usepackage{amsmath,amsfonts,bm}

\usepackage{graphicx}
\usepackage{caption}
\usepackage{subcaption}
\usepackage{booktabs} 
\usepackage{array}
\usepackage{makecell}

\DeclareMathOperator{\KL}{KL}

\usepackage{microtype}
\usepackage[export]{adjustbox}
\usepackage{multirow}
\usepackage{graphicx}
\usepackage{xcolor}
\usepackage{booktabs} 
\usepackage{url}            

\usepackage{amsfonts}       
\usepackage{nicefrac}       
\usepackage{microtype}      
\usepackage{bm}
\usepackage{diagbox}
\usepackage{color} 
\usepackage{amssymb}
\usepackage{amsmath}
\usepackage{amsfonts}
\usepackage{mathtools}
\usepackage{lineno}
\usepackage[toc,page,header]{appendix}
\usepackage{minitoc}

\newenvironment{sproof}{%
  \proof}{\endproof}

\newcommand{\Userpostexp}{R_{\mathrm{usr}}(\tilde Q)}
\newcommand{\Userpostemp}{\widehat R_{\mathrm{usr}}^{(M,N)}(\tilde Q)}

\newcommand{\sysexppre}{L^{\rm sys}_{P_{\mathrm{task}}(\cdot\mid\theta)}(G)}

\usepackage{mathrsfs}
\usepackage{amsthm}
\usepackage{amsmath}
\usepackage{amssymb, bm}
\usepackage{comment}
\usepackage{multirow}
\usepackage{tabularx}
\usepackage{multicol}
\usepackage[normalem]{ulem}
\usepackage{multicol, multirow}
\usepackage[ruled,vlined,algo2e]{algorithm2e}

\newcommand{\mc}{\mathcal}

\usepackage{algorithmic}
\usepackage{algorithm}

\newtheorem{theorem}{Theorem}

\newtheorem{assumption}{Assumption}
\newtheorem{condition}{Condition}
\newtheorem{proposition}{Proposition}
\newtheorem{corollary}{Corollary}
\newtheorem{lemma}{Lemma}

\newcommand{\Var}{\text{Var}}

\newcommand{\E}{\mathbb{E}}

\newcommand{\cX}{{\mathcal X}}

\SetKwInput{KwInput}{Input}                
\SetKwInput{KwOutput}{Output}

\providecommand{\customgenericname}{}
\newcommand{\newcustomtheorem}[2]{%
  \newenvironment{#1}[1]
  {%
  \renewcommand\customgenericname{#2}%
  \renewcommand\theinnercustomgeneric{##1}%
  \innercustomgeneric
  }
  {\endinnercustomgeneric}
}
\newcustomtheorem{customtheorem}{Theorem}

\usepackage{cleveref}
\usepackage{textcomp,    
            xspace}

\newcommand\blfootnote[1]{%
\begingroup
\renewcommand\thefootnote{}\footnote{#1}%
\addtocounter{footnote}{-1}%
\endgroup
}

\title{On the Limits of Prompt-Conditioned Language Models as General-Purpose Solvers}
\author{David Mguni$^{1}$, Julian Ma$^{2}$, Jun Wang$^{*2}$
\\ $^{1}$Queen Mary University London $^{2}$University College London}
\date{}
\begin{document}
\maketitle

\begin{abstract}
\blfootnote{$^*$Corresponding author  \textlangle junwang@cs.ucl.ac.uk\textrangle. }
Large Language Models (LLMs) are frequently portrayed as general-purpose solvers capable of solving arbitrary tasks. We argue that this view overlooks a fundamental constraint: language is a compressed and capacity-limited interface for conveying task information. Modelling User--System interaction as a bilevel \emph{cheap-talk} game, we analyse how latent tasks are encoded into prompts and reinterpreted under alignment and safety constraints. We introduce a conceptual decomposition separating task inference from execution and derive PAC-Bayes bounds that distinguish finite-sample estimation error from irreducible structural limitations. Our first main result establishes an \emph{expressivity floor}: language acts as a capacity-limited communication channel, and whenever the informational complexity of a task family exceeds the capacity of that channel, distinct tasks become unavoidably indistinguishable to the Solver, inducing a strictly positive error floor that cannot be eliminated by additional data, optimisation, or model scaling alone. We then establish an \emph{objective-misalignment floor}: when alignment constraints restrict the admissible output set, the User-ideal distribution may lie outside the feasible class, inducing an irreducible distortion. Together, these results yield a formal negative conclusion: prompt-conditioned LLMs are not universal problem solvers through prompting alone, as there exist task families for which correct behaviour is provably unattainable even in the infinite-data regime. More broadly, our analysis shows the limits of prompt-based generalisation arise from information-constrained communication and alignment-constrained objectives. This suggests that interfaces beyond natural language, including multimodal observations and, external memory, may reduce the inherent LLM limitations by increasing the task-relevant information available to the System.
\end{abstract}

\section{Introduction}


Large language models (LLMs) have rapidly become a general interface for computation, reasoning, and decision support, and are increasingly embedded across many industries~\cite{kaddour2023challenges,thirunavukarasu2023large, yuan2022wordcraft}. Despite the rapid pace of integration within various areas of technology, many of the core aspects that govern their behaviours and wider usages have not yet been understood \cite{frieder2023mathematical,tamkin2021understanding}. This leaves a deep gap in our understanding of LLMs and their capacity to tackle applications beyond traditional settings, for example, as general purpose learners. 

In standard training paradigms, LLMs are trained on large, fixed corpora and evaluated under supervised or self-supervised objectives. In contrast, real-world usage is prompt-driven: Users issue ad-hoc, task-specific requests expressed through free-form natural language. To succeed as general-purpose solvers, LLM-based systems must therefore infer a User’s latent task from an indirect and potentially ambiguous prompt, and then execute that task using pre-trained problem-solving capabilities.

This motivates an analytical decomposition of prompt-conditioned LLM behaviour into two conceptual components: a \emph{Prompt Interpreter}, which maps User prompts into an internal representation of intent under communication and normative constraints, and a \emph{Base Solver}, which executes the task conditional on this representation. This division need not be architectural and can be purely conceptual. Its utility lies in allowing us to isolate the informational and objective bottlenecks that arise in prompt-based interaction, without making claims about internal and physical modularity within the model (see Figure~\ref{fig:concept_summary}). Similar conceptual decompositions of LLM behaviour into separable functional roles without claiming architectural modularity appear in prior analyses of “language vs. thought” and agentic role-play views of language models (e.g., \citealp{mahowald2024dissociating, andreas2022language, shanahan2023role})
\begin{figure}
    \centering
    \includegraphics[width=.7\linewidth]{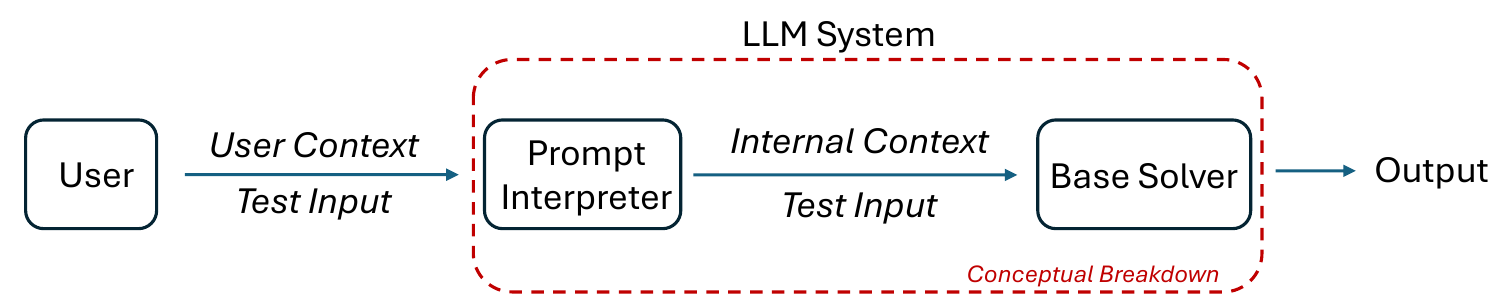}
    \caption{Conceptual Diagram}
    \label{fig:concept_summary}
\end{figure}

Concretely, we conceptually view prompting as a two-step channel. The User emits a natural-language context $c \sim \pi_U(\cdot \mid \theta)$, which is then rewritten or filtered into an internal context $\hat{c} \sim \pi_R^{\varphi}(\cdot \mid c, S)$ in the LLM. The Base Solver then produces $y \sim P_G(\cdot \mid x, \hat{c})$. Crucially, the Base Solver never observes $\theta$ directly but only via $\hat{c}$. Operationally, the effective pathway is the decomposition $\pi_R^{\varphi} \circ \pi_U$: task information must first be expressed in text and then survive System-side rewriting. When this channel is constrained, safety-driven, or normatively biased, the Base Solver will operate on a distorted representation of true intent, leading to a quantifiable gap between User intent and System output, as formalised by the misalignment terms introduced in our analysis. In Example 1, we illustrate the impact of finite communication channel capacity. Example 2 demonstrates the impact of objective misalignment.


\begin{tcolorbox}[
  breakable,
  colback=gray!5,
  colframe=gray!40!black,
  title=\textbf{Example 1: Clinical decision under compressed subjective reporting.}
]
\footnotesize
Consider a clinical setting in which a patient experiences a continuous pain severity 
\(
\Theta \in [0,1],
\)
where $\Theta=0$ denotes no pain and $\Theta=1$ denotes maximal severity. Relying on only text-based prompts from the patient, an LLM’s task is to infer the underlying physiological condition level $\theta$ in order to prescribe an appropriate intervention (e.g.\ dosage, imaging decision, escalation to surgery). 
Formally, we model the medically appropriate predictive distribution as
\[
P_\theta(y\mid x)
=
\mathcal{N}(y;\theta,\sigma^2),
\]
meaning that the optimal prediction of the clinically relevant outcome $y$ (such as required intervention intensity) is centred at the true condition $\theta$, with variance $\sigma^2$ representing irreducible biological uncertainty. 
Under log-loss, performance is measured by the KL divergence between the LLM’s predictive distribution and this ideal distribution. For two Gaussian distributions with equal variance,
\[
D_{\rm KL}\!\big(
\mathcal{N}(\theta,\sigma^2)
\;\|\;
\mathcal{N}(\mu,\sigma^2)
\big)
=
\frac{(\theta-\mu)^2}{2\sigma^2}.
\]
Therefore, the KL divergence reduces to a squared error in estimating the underlying condition $\theta$. 
A positive KL value represents systematic miscalibration in clinical judgement: the predicted intervention intensity is centred at the wrong physiological state. In practice, the LLM does not observe $\theta$ directly. 
The patient reports pain on a standard 10-point scale:
\[
C \in \{1,2,\dots,10\},
\qquad
C=b(\Theta):=\min\{10,\,1+\lfloor 10\Theta\rfloor\}.
\]
The reporting rule $b(\cdot)$ partitions $[0,1]$ into ten equal bins and returns the bin index. 
The floor function $\lfloor 10\Theta \rfloor$ extracts the integer part of $10\Theta$, so all severities within an interval of width $\Delta=0.1$ produce the same report. 
This is not dishonesty; it is structural compression of a continuous physiological variable into a coarse linguistic category.

Each distinct severity level $\theta$ defines a distinct clinical task, since it induces a distinct ideal predictive distribution $P_\theta$. 
However, once severity is mapped to $C$, all values of $\theta$ within the same bin become observationally indistinguishable. 
Two distinct condition–response pairs 
\[
(\theta_1, P_{\theta_1})
\quad\text{and}\quad
(\theta_2, P_{\theta_2})
\]
with $\theta_1\neq\theta_2$ but $b(\theta_1)=b(\theta_2)$ are therefore \emph{aliased}: the clinician (or LLM system) receives identical linguistic input despite the medically appropriate intervention differing. Suppose the LLM outputs a Gaussian prediction
\(
Q_c(y\mid x)=\mathcal{N}(y;\mu(c),\sigma^2).
\)
Conditional on observing $C=c$, the true severity $\Theta$ is uniformly distributed over an interval of width $\Delta=0.1$. 
The optimal prediction $\mu(c)$ is the midpoint of that interval. 
The residual conditional variance is
\[
\Var(\Theta\mid C=c)=\frac{\Delta^2}{12}.
\]
Substituting into the KL expression yields
\[
\inf_{\{Q_c\}}
\mathbb{E}\!\left[
D_{\rm KL}\!\big(
P_\Theta\|\;Q_{b(\Theta)}
\big)
\right]
=
\frac{\Delta^2}{24\sigma^2}
=
\frac{1}{2400\,\sigma^2}
>0.
\]

This shows that even with infinite patient data, perfect optimisation, and a highly capable model, the expected KL loss cannot vanish. 
Distinct underlying conditions that require meaningfully different interventions have been collapsed into the same linguistic category. 
The clinician’s predictive distribution must therefore compromise between aliased severities, inducing systematic miscalibration. 
The resulting error floor is neither statistical nor architectural; it arises from compressing a continuous state into a discrete language interface.

Now suppose that the patient provides
\emph{rich multimodal evidence}
\(
W = (C_{\mathrm{text}}, C_{\mathrm{img}}),
\)
consisting of:
(i) free-form textual symptom descriptions,
(ii) medical images (e.g.\ photographs of inflammation or rash),
and possibly additional structured signals.
The raw multimodal object $W$ may contain highly informative,
fine-grained information about $\Theta$. However, modern multimodal LLM systems do not operate directly on $W$.
Instead, they construct an \emph{internal representation}
\[
Z = \pi_R(W),
\]
where $\pi_R$ denotes the model’s preprocessing and rewriting mechanism.
In practice, this transformation includes vision encoders mapping images into fixed-dimensional embeddings, projection of embeddings into token sequences, context-window truncation, summarisation to satisfy length constraints, safety or policy-driven rewriting and formatting into structured textual prompts. Hence, all downstream prediction depends only on $Z$:
\[
Y \sim Q(\cdot \mid x, Z).
\]
The full information pathway is therefore
\(
\Theta \longrightarrow W \longrightarrow Z \longrightarrow Y,
\)
which forms a Markov chain.
In particular, conditional on $Z$, the output $Y$ is independent of the raw multimodal evidence $W$ and of the latent task $\Theta$. By the Data Processing Inequality,
\(
I(\Theta;Z)
\;\le\;
I(\Theta;W)
\). Moreover, because $Z$ is a finite-dimensional internal representation
(e.g.\ fixed-length embeddings or bounded token sequences),
its mutual information with $\Theta$ is bounded by the representational capacity
of the architecture:
\[
I(\Theta;Z) \;\le\; B_{\mathrm{mdl}},
\]
for some finite $B_{\mathrm{mdl}}$ determined by embedding dimension,
context window length, tokenisation constraints,
and architectural bottlenecks. Suppose now that the task family contains $K$ distinct severity levels
$\theta_1,\dots,\theta_K$,
each inducing a distinct predictive distribution
$P_{\theta_j}$,
and that the tasks are predictively separated:
\[
\min_{i\neq j}
\E_x\!\left[
D_{\rm KL}\!\big(
P_{\theta_i}(\cdot\mid x)
\;\|\;
P_{\theta_j}(\cdot\mid x)
\big)
\right]
\;\ge\; \delta > 0.
\]

If the representational capacity of the architecture is sufficiently constrained, then the task cannot be identified from $Z$ with probability one.
Consequently, distinct physiological states must be mapped to the same
internal representation with positive probability.
This induces \emph{task aliasing}:
\[
\theta_1 \neq \theta_2
\quad\text{but}\quad
\pi_R(W_1) = \pi_R(W_2) = Z.
\]

Conditional on observing $Z=z$,
the latent severity $\Theta$ remains random.
The optimal prediction is therefore
\[
\mu(z) = \E[\Theta \mid Z=z],
\]
and the residual conditional variance
\[
\Var(\Theta \mid Z=z)
\]
is strictly positive whenever aliasing occurs.
Substituting into the Gaussian KL expression yields
\[
\inf_Q
\E\!\left[
D_{\rm KL}\!\big(
P_\Theta \,\|\, Q(\cdot\mid Z)
\big)
\right]
=
\frac{1}{2\sigma^2}
\E\!\left[
\Var(\Theta \mid Z)
\right]
\;>\; 0.
\]

Hence, this example shows that even though the patient supplies rich multimodal evidence,
the model must compress this evidence into a bounded internal state
before acting.
If the architectural capacity of this representation is insufficient
to uniquely encode the underlying task,
distinct clinical conditions become aliased.
The resulting error floor is therefore not statistical and not due to
imperfect optimisation.
It arises from representational compression inside the model itself. This illustrates that even multimodal LLM systems,
when constrained to finite internal representations,
can exhibit irreducible task confusion induced by architectural
information bottlenecks.
\end{tcolorbox}

This abstraction also aligns with empirical LLM observations, where general problem-solving ability is primarily acquired during large-scale pre-training, while behavioural shaping are strongly influenced by subsequent fine-tuning and alignment stages. Relatively small parameter updates during SFT or RLHF can substantially alter model behaviour \citep{zhang2024llama, taori2023stanford, hu2022lora, ouyang2022training}. Such effects can be understood as consequences of communication and objective constraints rather than changes raw model capability.

Building on this construct, we further model User--System interaction using an information-theoretic analogue of Bayesian cheap talk. The User privately observes the task, the System does not, and their objectives need not coincide. This perspective allows us to study alignment, information loss, and generalisation in a unified way, without appealing to equilibrium assumptions.

Our analysis covers both regimes in which the Prompt Interpreter can reliably encode the latent task and where communication or objective constraints prevent faithful task transmission. We characterise when prompt-conditioned learning succeeds, and when fundamental limits persist independently of data availability. Our analysis yields two population-level impossibility results and a finite-sample generalisation guarantee. First, we prove an expressivity/identifiability floor: if the effective task information available to the Base Solver is bounded,
\(I(\Theta;Z\mid \Theta\in\Theta_0)\le B\), relative to the entropy of a task family, then task aliasing is unavoidable and the User risk admits a strictly positive asymptotic lower bound (Theorem 1). Second, we prove an objective/admissibility floor: if alignment or safety constraints restrict the System to an admissible set $P_{safe}(x,c)$, then whenever the User-ideal distribution lies outside this set on a non-negligible region, a positive KL gap persists even with perfect task inference (Theorem 2). Finally, we derive PAC–Bayes meta-learning bounds that separate vanishing estimation terms from non-vanishing structural floors when estimating population risk from sample risk (Theorems 3–5). In summary, we derive (i) a population floor—limits on achievable performance even with perfect optimisation—and (ii) analyse estimation limits—limits on how confidently we can infer population safety/performance from finite audits.

Our results highlight the relevant bottleneck to be the \textit{effective task information} that survives the full interface $\pi_R^{\varphi} \circ \pi_U$ and conditions the Base Solver. An implicit implication is that communication channels with higher capacity, such as those of multi-modal inputs, may alleviate the limitation attributable to language-based prompt compression.

\begin{tcolorbox}[
  breakable,
  colback=gray!5,
  colframe=gray!40!black,
  title=\textbf{Example 2 (objective misalignment): a benign clinical task with a positive KL floor.}
]
\footnotesize 
Consider a clinical risk-stratification task where the output is a calibrated probability of whether urgent escalation is needed. Let the output space be $\mathcal Y=\{0,1\}$, where $Y=1$ denotes \emph{urgent escalation} and $Y=0$ denotes \emph{non-urgent}. For a (benign) task $\theta$ and input $x\in\mathcal X$, the User-ideal predictive distribution is
\[
P^{\mathrm{usr}}_\theta(\cdot\mid x)=\mathrm{Bernoulli}\big(p_\theta(x)\big),
\qquad
p_\theta(x):=P^{\mathrm{usr}}_\theta(Y=1\mid x).
\]
Therefore, the User intends the System to output a clinically meaningful probability $p_\theta(x)$ from which a clinician can select an appropriate intervention policy.

Now suppose the System enforces an admissibility (safety/alignment) constraint that \emph{caps} how strongly it may predict urgent escalation under certain contexts $c$ (e.g.\ conservative policy mode, missing structured evidence, or a generic disclaimer regime). Let us model this as an admissible set
\[
\mathcal P_{\mathrm{safe}}(x,c)
=
\Big\{\mathrm{Bernoulli}(q):\ q\le q_{\max}(c)\Big\},
\qquad q_{\max}(c)\in(0,1).
\]
The objective distortion at $(\theta,x,c)$ is the distance from the User-ideal distribution to the admissible set:
\[
\Delta_{\mathrm{obj}}(\theta,x,c)
:=
\inf_{Q\in \mathcal P_{\mathrm{safe}}(x,c)}
D_{\rm KL}\!\Big(P^{\mathrm{usr}}_\theta(\cdot\mid x)\,\Big\|\,Q\Big)
=
\inf_{q\le q_{\max}(c)}
D_{\rm KL}\!\Big(\mathrm{Bern}(p_\theta(x))\,\Big\|\,\mathrm{Bern}(q)\Big).
\]
Since $D_{\rm KL}(\mathrm{Bern}(p)\|\mathrm{Bern}(q))$ is minimised at $q=p$ when feasible, the minimiser is the \emph{closest admissible} probability
\[
q^\star(x,c)=\min\{p_\theta(x),\,q_{\max}(c)\}.
\]
Therefore, whenever $p_\theta(x)>q_{\max}(c)$,
\[
\Delta_{\mathrm{obj}}(\theta,x,c)
=
D_{\rm KL}\!\Big(\mathrm{Bern}(p_\theta(x))\,\Big\|\,\mathrm{Bern}(q_{\max}(c))\Big)
\;>\;0,
\]
i.e.\ even a perfect optimiser \emph{within} the admissible class cannot match the User-ideal predictive distribution. To see an explicit positive population floor, assume there exist constants $p_0>q_0$ and $\beta\in(0,1]$ such that with probability at least $\beta$ over $(\Theta,X,C)$,
\[
p_\Theta(X)\ \ge\ p_0
\quad\text{and}\quad
q_{\max}(C)\ \le\ q_0.
\]
For example: a non-negligible subset of benign but truly high-risk cases, combined with a non-negligible probability that the System enters a conservative policy regime.
Then on this event,
\[
\Delta_{\mathrm{obj}}(\Theta,X,C)
\ \ge\
D_{\rm KL}\!\Big(\mathrm{Bern}(p_0)\,\Big\|\,\mathrm{Bern}(q_0)\Big)
\;>\;0,
\]
so that given $\varepsilon:=D_{\rm KL}\!\Big(\mathrm{Bern}(p_0)\,\Big\|\,\mathrm{Bern}(q_0)\Big)>0$, we have that $\mathbb P(\Delta_{\mathrm{obj}}(\Theta,X,C)\ge \varepsilon)\ge \beta$. 

\noindent\textbf{Interpretation.} The task itself is benign (clinical calibration), but admissibility constraints can still force the System’s predictive distribution to be \emph{too conservative} on a non-negligible set of inputs/contexts. The resulting \emph{distance-to-admissible-set} $\Delta_{\mathrm{obj}}$ is a quantitative measure of objective misalignment: it is the irreducible KL gap between what the clinician (User) needs the model to express and what the aligned System is permitted to output.
\end{tcolorbox}

\section{Main Results}

Despite extensive empirical study, the sample efficiency and generalisation limits of prompt-conditioned language models remain poorly characterised. While prior work has examined computational expressivity and complexity-theoretic limitations of autoregressive models \cite{lin2020limitations, floridi2020gpt}, comparatively little is known about how reliably such models can infer and execute latent tasks from indirect communication.

This motivates the central question of the paper: \emph{to what extent can prompt-conditioned language models reliably infer and execute a User’s latent task?} From a learning-theoretic perspective, successful generalisation requires not only sufficient representational and computational capacity, but also accurate task inference from constrained, free-form prompts. This inference problem is intrinsic to prompt-based interaction and cannot be reduced to standard supervised learning.

The paper makes two main claims. First, even granting a highly capable Base Solver, we show there can remain an irreducible population-level performance floor induced by upstream objective mismatch and communication limits. Second, we quantify how reliably empirical measurements predict population performance, via PAC–Bayes generalisation bounds that control the empirical–population gap.

Our negative result is also consistent with the No Free Lunch perspective: without a sufficiently informative task signal (in our case, prompt communication through a constrained communication channel) to invoke the correct inductive bias for the relevant task family, no single solver can achieve uniformly low error across all tasks. In our setting, this manifests as an irreducible User-risk floor driven by limited task transmission and admissibility constraints, even when the Base Solver is assumed powerful.


We establish two complementary classes of results that together characterise the limits of prompt-conditioned general-purpose learning.

\textbf{Population-level misalignment limits.}
We first derive lower bounds on the User’s population risk that hold independently of sample size and model optimisation. We show that when either (i) the System’s admissible outputs are misaligned with the User’s intended task or (ii) the prompt and interpretation mechanism has bounded information capacity, the User risk admits a strictly positive asymptotic floor (Theorems~\ref{thm:blended_floor_usr} and~\ref{thm:misalignment_flr_meta}). These bounds formalise misalignment as a structural property of communication and admissibility constraints, rather than as an estimation or optimisation artifact.

\textbf{PAC-Bayes generalisation and estimation error.}
We then derive PAC--Bayes generalisation bounds that quantify how empirical User risk concentrates around its population counterpart as the number of observed tasks and within-task samples increases (Theorems~\ref{lem:task_pb} and~\ref{thm:user_meta_pb}). These bounds cleanly separate vanishing estimation error from non-vanishing structural error: complexity terms decay with data, while misalignment-induced floors persist. In the fully revealing case, task aliasing disappears and the communication-induced obstacle to zero-shot and few-shot adaptation is removed. Any remaining error must therefore arise from limitations of
the solver itself rather than from ambiguity in task identification.

Taken together, these results show that PAC--Bayes analysis does not eliminate misalignment, but instead exposes how finite-sample learning behaviour converges to fundamental communication- and objective-induced limits. In particular, we prove that any prompt-conditioned system with finite interpretative capacity admits task families for which correct output generation is provably unattainable, regardless of the amount of data or number of in-context examples. Beyond exposing the limitation of these systems, our results show that expanding capability in practice may require expanding and stabilising the task interface and not merely scaling the model or dataset size.

We interpret objective misalignment through a structural analogy to cheap-talk models, where senders and receivers optimise different objectives \citep{crawford1982strategic, farrell1996cheap}. Representational misalignment corresponds to information bottlenecks in the communication channel. Each induces a distinct and additive contribution to the generalisation limits of prompt-based learning. 

Our analysis is grounded in PAC--Bayes theory \citep{alquier2021User} and meta-learning \citep{rezazadeh2022unified}, but departs from existing work in three key respects: we allow latent task misspecification, permit genuine objective divergence between User and System, and work with unbounded log-loss objectives. Together, these features yield a principled characterisation of when prompt-conditioned language models can generalise efficiently and when fundamental limits persist.
\subsection*{Preliminaries}

We use the following notation which is common in the PAC-Bayes literature  \cite{rezazadeh2022unified,masegosa2020learning} (see Table \ref{tab:concept_map_1} and \ref{tab:concept_map}). 

We first take a brief excursion to review some key concepts in machine learning \cite{greene2008unsupervised}. These elements will serve as an important basis for our subsequent analysis. A latent task (or intent) $\theta \sim P_{\mathrm{meta}}$ induces a task-conditional distribution over interaction records
\(s=(x,y)\in\mathcal X\times\mathcal Y\), where $x$ denotes the model input (user query or prompt plus any provided context) and $y$ denotes the target response or completion.
For a fixed $\theta$, we observe a finite dataset $S_M=\{(x_i,y_i)\}_{i=1}^M$ drawn i.i.d.\ from $P_{\mathrm{task}}(\cdot\mid\theta)$; depending on the setting,
$S_M$ can represent a training corpus, a benchmark or evaluation suite, or an audit trace.

A system induces a conditional distribution $P_h(\cdot\mid x)$ over outputs given inputs; we call any such conditional distribution a \emph{predictor} $h$,
and we write $\mathcal{H}$ for an admissible family of predictors (playing the role of a hypothesis class).
Using negative log-likelihood (cross-entropy) loss, the population risk under task $\theta$ is
\[
\mathrm{er}(h;\theta)
:=\mathbb{E}_{(x,y)\sim P_{\mathrm{task}}(\cdot\mid\theta)}\big[-\log P_h(y\mid x)\big],
\]
with empirical counterpart $\widehat{\mathrm{er}}(h;S_M)$ defined by sample averaging.
This notation is standard in PAC--Bayes; below we refine $P_h$ to be the particular system components and constraints relevant to LLM deployment, and we will distinguish
between user-facing and system-facing performance criteria.

\begin{table}[t]
\centering
\small
\caption{Key concepts in latent task and data generation}
\label{tab:concept_map_1}
\renewcommand\arraystretch{1.2}
\setlength{\tabcolsep}{6pt}
\begin{tabularx}{\linewidth}{>{\raggedright\arraybackslash}p{3cm} >{\raggedright\arraybackslash}p{3cm} X}
\toprule
Concept & Symbol & Role / LLM interpretation \\
\midrule

\multicolumn{3}{l}{\textbf{Latent task and data-generating process}}\\
Task (latent) & $\theta$ & Latent task or intent sampled from \(P_{\rm meta}\). \\
Meta-task distribution & $P_{\mathrm{meta}}(\theta)$ & Distribution over tasks governing meta-training and evaluation across tasks.
\\
Support information & $S$ & Available support information (e.g.\ prior interactions, 
system memory, or contextual metadata).\\
Task data and behavioural prompt
&
\(S_\theta^{(N)}\sim P_{\mathrm{task}}(\cdot\mid\theta)^N,\;
c\sim\pi_U(\cdot\mid\theta)\)
&
For each task, data are sampled from the task-conditional law, while the User emits a prompt according to a task-dependent behavioural communication policy. \\
Joint observable distribution &
\(\bar P(dx,dy,dc)
:=\int_{\Theta}
P_{\mathrm{task}}(dx,dy\mid\theta)\,
\pi_U(dc\mid\theta)\,
P_{\mathrm{meta}}(d\theta)\) &
Joint distribution over inputs, outputs, and contexts after marginalising over latent tasks. \\

\bottomrule
\end{tabularx}
\end{table}

\begin{table}[t]
\centering
\small
\caption{Mapping between key concepts in our framework and common LLM usage}
\label{tab:concept_map}
\renewcommand\arraystretch{1.2}
\setlength{\tabcolsep}{6pt}
\begin{tabularx}{\linewidth}{>{\raggedright\arraybackslash}p{3cm} >{\raggedright\arraybackslash}p{3cm} X}
\toprule
Concept & Symbol & Role / LLM interpretation \\
\midrule

\multicolumn{3}{l}{\textbf{User communication}}\\
User prior & \(P_U(c)\) & Reference distribution over contexts used in the User PAC--Bayes change-of-measure arguments. \\
Behavioural prompting kernel & \(\Pi(c\mid\theta)\) or \(\pi_U(c\mid\theta)\) & Task-conditional communication policy used to analyse task information and revelation. \\
PAC--Bayes User posterior & \(\tilde Q(c)\) & Optimisation-induced posterior over contexts used in the User PAC--Bayes bounds. \\
User context & \(c\in\mathcal C\) & External prompt or communication signal. \\

\midrule
\multicolumn{3}{l}{\textbf{System components and induced prediction}}\\
Prompt Interpreter & $R$ & Given User-provided prompt $c$ generates $\hat c$ used by Base Solver; encodes how User prompts (context) are interpreted by core model. \\
Internal Context & $\hat c \sim \pi_R(\cdot \mid c,S)$ & Context generated by $R$ after interpreting $c$. \\
Base Solver &
$P_G(y\mid x,\hat c)$ & Conditional solver producing output $y$ given input $x$ and Internal Context $\hat c$; parameterised by $G \in \mathcal{G}$ and trained during unsupervised pre-training. \\
System-induced predictive distribution & $P_{\varphi}(y \mid x,c,S) = \mathbb{E}_{\hat c \sim \pi^{\varphi}_R}[P_G(y \mid x,\hat c)]$ & Effective conditional distribution induced by the Prompt Interpreter and Base Solver given User context $c$ and support set $S$. \\
System prior & $P_0(G)$ & Prior over model parameters $G$. \\
System posterior & $Q(G \mid S)$ & Data-dependent System posterior over $G$. \\

\midrule
\multicolumn{3}{l}{\textbf{Objectives and misalignment}}\\
User Loss & $L_{\mathrm{usr}}$ & User objective, defined through $P_\psi(y \mid x,c)$. \\
System Loss & $L_{\mathrm{sys}}$ & Objective the core model is optimised for (may differ from $L_{\mathrm{usr}}$, e.g.\ when User task deemed unsafe). \\
Alignment Loss (gap) & $L_{\mathrm{align}}$ & Performance penalty under misspecification / misalignment between User and Prompt Interpreter. \\

\bottomrule
\end{tabularx}
\end{table}

\section{Using LLMs for General-Purpose Learning}

Pre-training the System. We consider a LLM where $\theta \in \Theta$ acts as a hidden
task or topic \cite{xie2021explanation}. The task $\theta$ is sampled from the meta-distribution \(P_{\rm meta}\). Let us begin with pre-training a language model $P_G$ parameterised by
$G \in \mathcal{G}$. For a given input $x \in \mathcal{X}$, the model generates
an output $\hat y$ according to $
\hat y \sim P_G(y \mid x)$ where the marginal predictive distribution given $x$ can be written as $
P_G(y \mid x)
= \int_{\Theta} P_G(y \mid x,\theta)\,P_{\rm meta}(d\theta)$
and the task-specific data distribution is $
s = (x,y) \sim P_{\mathrm{task}}(\cdot \mid \theta), s \in \mathcal{X}\times\mathcal{Y}$. 
  The objective is to minimise some loss function given some choice of inference model parameter. For now, assume the model space $\mathcal{G}$ is correctly specified that is, we assume there exists a $ G' \in \mathcal{G}$ which specifies the true model. The System (i.e. the language model and its procedure for performing its updates) receives a prefixed feature $x$ and generates $\hat{y}$ according to distribution $
\hat{y} \sim  \operatorname{softmax} P_{G}(y \mid x)$. 

In general, the goal of training is to infer a model parameter $G'\in \mathcal{G}$ that minimises the error in predicting the correct output, i.e. that which enables $\hat{y}$ to be accurately predicted from any given $x$.
This reproduces the setup in standard supervised learning where given some loss function or `risk' we seek to find a hypothesis among a fixed class of functions or hypothesis space that minimises the loss function. In our case, the objective of the System is to find a $G'$ such that $
G'\in\underset{G\in\mathcal{G}}{\operatorname{argmin}}\; \sysexppre$,
where the objective $\sysexppre$ is described by the expected cross-entropy loss function:

\begin{equation}
{\mathrm CE}(G)  :=-\mathbb{E}_{s\sim P_{\mathrm{task}}(\cdot\mid\theta)}\left[\log P_{G}(y \mid x) \right], \qquad \forall G
\in\mathcal{G}. 
\label{llm_1}    
\end{equation}
Since $P_{\mathrm{task}}(\cdot\mid\theta)$ is a priori unknown, given the observations $\boldsymbol{S}^M_\theta:=\{(x_i,y_i)\}_{i=0}^{M<\infty}$, we consider the empirical loss i.e we take an average of the cross-entropy loss: $
{\widehat{\mathrm CE}^{(N)}}  
 :=-\frac{1}{N} \sum_{i=1}^N \log P_{G}(y_i \mid x_i) ,\;s_i\equiv (x_i,y_i)\sim P_{\mathrm{task}}(\cdot\mid\theta)$ for any $G
\in\mathcal{G}$ and any $\theta\in\Theta$, where the sum over $i$ iterates through all sequences in the dataset $\boldsymbol{S}^M_\theta$. A crucial point to note is that the System loss function is unbounded.
\subsection{General Purpose Learners with Context Communication}


For each latent task $\theta \in \Theta$, let $(x,y)\sim P_{\mathrm{task}}(\cdot\mid\theta)$ denote data drawn from the task-level data-generating distribution.
We distinguish this from a meta-distribution $P_{\mathrm{meta}}$ over tasks, which governs how tasks are sampled during training and evaluation of the User and Prompt Interpreter.

We now consider observations of the data set $\boldsymbol{Z}^n_\theta=(z^1_\theta,\ldots,z^N_\theta)\in \boldsymbol{\mathcal{Z}}^n$ where $z^i_\theta = (x_i,c)\in  X\times \mathcal{C}\equiv\boldsymbol{\mathcal{Z}}$ 
and $N$ is the integer number of sequences in the dataset where now the input space which previously consisted of just $x\in X$ has been augmented to include a `context' (User Context) $c\in\mathcal{C}$ which is the output of a map $\mathfrak{c}:\Theta\to\mathcal{C}$ where $\mathcal{C}$ is a set of \textit{contexts}. The System receives the context \(c\) (along with a prefixed feature \(x\)) and generates 
\(\hat y\) according to the conditional distribution $
\hat y \sim P_G(y \mid x,\hat c)$, 
where \(\hat c\) is produced by a \emph{Prompt Interpreter} from the User's context \(c\). 
 
Formally, the System’s predictive mechanism is the conditional distribution
$P_{G}(y \mid x, \hat c)$ obtained after the Prompt Interpreter produces an Internal Context
$\hat c \sim \pi_{R}(\cdot \mid c, S)$. The introduction of Prompt Interpreter requires no explicit architectural separation; 
this factorization is conceptual: it is an analytical device for separating intent-to-objective translation and communication constraints from task execution, not a claim that any model internally implements clean modular subcomponents. Utilizing this conceptual decomposition isolates the interpretive role from the general problem solving role and is essential to our analysis.

For each task $\theta$, only a finite dataset is available
\[
S_\theta^{(N)}=\{(x_i,y_i)\}_{i=1}^N \sim P_{\mathrm{task}}(\cdot\mid\theta)^N 
\]
The User observes $\theta$ and produces a communication context $c\in\mathcal C$ according to a (possibly stochastic) policy $\pi_U(c\mid\theta)$.
This context is processed by the Prompt Interpreter, modelled as a stochastic rewriting mechanism
\[
\hat c \sim \pi_R^\varphi(\cdot\mid c,S_\theta^{(N)}),
\]
parameterised by $\varphi$, yielding an internal context $\hat c\in\mathcal C_{\mathrm{int}}$.
The Base Solver, parameterised by $G$, then produces predictions according to the conditional distribution
\[
y \sim P_G(\cdot\mid x,\hat c).
\]

For a given internal context $\hat c \in \mathcal C_{\mathrm{int}}$ and input $x \in \mathcal X$, the Base Solver is represented by the conditional predictive distribution $
P^G(y \mid x, \hat c)$. For a given external context $c \in \mathcal C$ and task-specific dataset $S_\theta^{(N)}$, define the effective conditional as
\begin{align}
P_{G,\varphi}(y \mid x,c,S_\theta^{(N)})
\;:=\;
\mathbb E_{\hat c \sim \pi_R^{\varphi}(\cdot\mid c,S_\theta^{(N)})}
\big[
P_G(y \mid x,\hat c)
\big]    
\end{align}
which marginalises over the Prompt Interpreter's rewriting mechanism. 

Let
\(
S_\theta^{(N)}
=
\{(x_i,y_i)\}_{i=1}^N
\sim
P_{\mathrm{task}}(\cdot\mid\theta)^N
\)
denote the support dataset. For notational convenience, once the support-set size
\(N\) is fixed we write \(S_\theta\) in place of
\(S_\theta^{(N)}\). Throughout the remainder of the paper,
\(S_\theta\) therefore denotes a realised support dataset
drawn from \(P_{\mathrm{task}}(\cdot\mid\theta)^N\).

\paragraph{Loss Formulation.} For a task $\theta\in\Theta$ and context $c\in\mathcal C$, we define the User and System population losses as 
\begin{align}
\ell_{\mathrm{sys}}(x,y;\theta,c,G,\varphi)\label{eq:effective_system_loss}
\;&:=\;
-\log P_{G,\varphi}(y \mid x,c,S_\theta^{(N)}), 
\\
\ell_{\mathrm{usr}}(x;\theta,c,S_\theta)
&:=
D_{\rm KL}
\!\left(
P_\theta(\cdot\mid x)
\,\middle\|\,
P_{G,\varphi}(\cdot\mid x,c,S_\theta)
\right). \label{eq:user_emp_loss_KL_ratio}
\end{align}
where $P_\theta(\cdot\mid x)$ denotes the ideal task-conditioned predictive distribution associated with task $\theta$, representing the User’s intended behaviour in the absence of communication or alignment constraints. 
Note, we distinguish between the task data distribution $P_{\mathrm{task}}(x,y\mid\theta)$, which governs how inputs and outputs are sampled, and the ideal task-conditioned predictive distribution $P_\theta(y\mid x)$, which specifies the User’s intended behaviour for task $\theta$. 
The ideal task-conditioned distribution $P_\theta(\cdot\mid x)$ encodes the User’s intended behaviour for task $\theta$, which may differ from the data-generating distribution $P_{\mathrm{task}}(\cdot\mid\theta)$ due to safety constraints, filtering, or other forms of misalignment.

The expected task-conditional losses are
\begin{align}
\mathcal L_{\mathrm{sys}}(\theta;c)
&=
\mathbb E_{(x,y)\sim P_{\mathrm{task}}(\cdot\mid\theta)}
\big[
\ell_{\mathrm{sys}}(x,y;\theta,c,G ,\varphi)
\big],
\label{eq:sys_expected_loss}
\\
\mathcal L_{\mathrm{usr}}(\theta;c\mid S_\theta)
=
\mathbb E_{x\sim\mu_\theta}
\big[
\ell_{\mathrm{usr}}(x;\theta,c,S_\theta)
\big].
\label{eq:user_pop_loss_KL}
\end{align}

Differences between
\(\mathcal L_{\mathrm{usr}}\)
and
\(\mathcal L_{\mathrm{sys}}\)
quantify objective misalignment, while limitations of the admissible context space \(\mathcal C\) and the rewriting mechanism \(\pi_R^\varphi\) give rise to expressivity misalignment.
Given a finite dataset $S_\theta^{(N)}$, the empirical System loss is
\begin{equation}
\widehat{\mathcal L}_{\mathrm{sys}}^{(N)}(\theta;c,G,\varphi)
:=
-\frac{1}{N}
\sum_{(x_i,y_i)\in S_\theta^{(N)}}
\log
P_{G,\varphi}(y_i\mid x_i,c,S_\theta^{(N)}).
\label{eq:sys_empirical_loss}
\end{equation}
This empirical loss is the negative log-likelihood of the effective predictive
distribution \(P_{G,\varphi}\), after marginalising over the Prompt Interpreter's
rewriting randomness. The population (task) loss is
\begin{equation}
\mathcal L_{\mathrm{sys}}(G , \varphi;\theta,c)
:=
\mathbb E_{(x,y)\sim P_{\mathrm{task}}(\cdot\mid\theta)}
\big[
\ell_{\mathrm{sys}}(G , \varphi;c,S_\theta^{(N)},x,y)
\big].
\label{eq:pop_task_sys_loss_skeleton}
\end{equation}

For the User-risk estimator, let \(S_\theta\) denote the support information
available to the Prompt Interpreter, and let
\[
T_\theta^{(N)}=\{(x_i,y_i)\}_{i=1}^N
\]
be an independent evaluation sample drawn from the User-ideal task distribution:
\[
x_i\sim \mu_\theta,
\qquad
y_i\sim P_\theta(\cdot\mid x_i).
\]
The empirical User loss is
\begin{align}
\widehat{\mathcal L}_{\mathrm{usr}}^{(N)}
(\theta;c\mid T_\theta^{(N)},S_\theta)
=
\frac1N
\sum_{i=1}^{N}
\log
\frac{
P_\theta(y_i\mid x_i)
}{
P_{G,\varphi}(y_i\mid x_i,c,S_\theta)
}.\label{eq:user_emp_KL_ratio}
\end{align}
Its population counterpart, conditional on the support information \(S_\theta\), is
\[
\mathcal L_{\mathrm{usr}}(\theta;c\mid S_\theta)
=
\mathbb E_{x\sim\mu_\theta}
\left[
D_{\rm KL}
\!\left(
P_\theta(\cdot\mid x)
\,\middle\|\,
P_{G,\varphi}(\cdot\mid x,c,S_\theta)
\right)
\right].
\]
Then
\[
\mathbb E_{T_\theta^{(N)}}
\left[
\widehat{\mathcal L}_{\mathrm{usr}}^{(N)}
(\theta;c\mid T_\theta^{(N)},S_\theta)
\right]
=
\mathcal L_{\mathrm{usr}}(\theta;c\mid S_\theta).
\]









For readability, we suppress the conditioning on \(S_\theta\) and write
\(\mathcal L_{\mathrm{usr}}(\theta;c)\) when the support information is fixed.

Note that the dependence of $\widehat{\mathcal L}_{\mathrm{usr}}^{(N)}(\theta;c\mid T_\theta^{(N)},S_\theta)$ on the context $c$ is entirely through the System likelihood term.
In particular, for a given task $\theta$ and dataset $S_\theta^{(N)}$ and for any $c\in\mathcal C$ we have that
\begin{equation}
\widehat{\mathcal L}_{\mathrm{usr}}^{(N)}(\theta;c\mid T_\theta^{(N)},S_\theta)
=
\frac{1}{N}\sum_{i=1}^N \log P_\theta(y_i\mid x_i)
-
\frac{1}{N}\sum_{i=1}^N
\log P_{G,\varphi}(y_i\mid x_i,c,S_\theta).
\label{eq:usr_KL_decomp}
\end{equation}
The first term depends only on \((\theta,T_\theta^{(N)})\) and not on \(c\). Consequently, any minimiser of the empirical User loss over contexts also maximises the empirical System likelihood:
\begin{equation}
\arg\min_{c\in\mathcal C}\;
\widehat{\mathcal L}_{\mathrm{usr}}^{(N)}(\theta;c\mid T_\theta^{(N)},S_\theta)
\equiv
\arg\max_{c\in\mathcal C}\;
\frac{1}{N}\sum_{i=1}^N
\log P_{G, \varphi}(y_i\mid x_i,c,S_\theta^{(N)}),
\label{eq:usr_context_argmin}
\end{equation}
where the decomposition \eqref{eq:usr_KL_decomp} follows by expanding the log-ratio in \eqref{eq:user_emp_KL_ratio}.
The first term depends only on $(\theta,S_\theta^{(N)})$ and not on $c$.
Therefore minimising \eqref{eq:user_emp_KL_ratio} over $c$ is equivalent to maximising the second term, yielding \eqref{eq:usr_context_argmin}. This shows that although $\mathcal L_{\mathrm{usr}}(\theta;c)$ is defined as a KL divergence to the task distribution $P_\theta(\cdot\mid x)$, the induced empirical optimisation over contexts depends only on the System likelihood term.
This permits direct PAC--Bayes generalisation bounds for the User objective while retaining a KL-based population criterion.

The KL formulation in
~\eqref{eq:user_emp_loss_KL_ratio} serves as a population-level criterion that quantifies
misalignment between the System’s induced behaviour and the User’s ideal task distribution. The loss in \eqref{eq:effective_system_loss} is defined as the negative log-likelihood of the \emph{effective} predictive distribution obtained after marginalising over the Prompt Interpreter.
This corresponds to treating the internal context $\hat c$ as a latent variable rather than as an observed decision.
By Jensen’s inequality, this choice yields a tighter and operationally more meaningful objective than the expected negative log-likelihood of individual rewrites, and reflects the behaviour of the System as a mixture model from the perspective of the User. In practice, both losses are unbounded, so the subsequent generalisation analysis (Section \ref{sec:generalisation}) relies on sub-Gaussian concentration assumptions rather than bounded-loss inequalities. 

\paragraph{The role of \(N\) in the User and System losses.}
The role of N in the User and System losses is to index a within-task sample size used to approximate task-conditioned expectations by empirical averages. Formally, N controls how precisely quantities such as cross-entropies or KL-style risks are estimated from samples drawn under the task distribution. In particular, the empirical User risk is introduced as a statistical estimator of the population User risk; it is not meant to imply that a human user literally observes or optimizes over N samples during interaction. We only require that user-facing tasks and system-facing tasks are i.i.d. draws from the same meta-distribution, not that the same task must be replayed counter-factually for the user.

\textbf{Notation convention (training vs. evaluation).} We use $N_{train},M_{train}$, for sample sizes that index optimisation or training objectives, and $N_{eval},M_{eval}$ for sample sizes that index empirical-to-population estimation in generalisation bounds. For readability we set $N_{train} = N_{eval} = N$ and $M_{train} = M_{eval} = M$

We now lift the task-level losses to a meta-learning setting, where both the System and the User optimise distributions over hypotheses and contexts across tasks.

\paragraph{Three distributions.}
Throughout the paper we distinguish three conceptually different
probability laws:

\begin{itemize}
\item $P_\theta(\cdot\mid x)$:
the User-ideal predictive law associated with task $\theta$
(the behaviour the user actually wants);

\item $P_{\mathrm{task}}(\cdot\mid\theta)$:
the task/environment data-generating law from which training or
evaluation datasets may be sampled;

\item $P_{G,\varphi}(\cdot\mid x,c,S_\theta^{(N)})$:
the predictive law induced by the User--System pair
after prompting, rewriting, and inference.
\end{itemize}

The User-level risks compare $P_{G,\varphi}$ against
$P_\theta$, whereas empirical estimators may be constructed from
samples generated according to either $P_\theta$ or
$P_{\mathrm{task}}$, depending on the evaluation protocol.

\paragraph{Meta-Loss Formulation.} We assume a given fixed hypothesis class $\mathcal{H}$ and a given context set $\mathcal{C}$. 
We consider a three-level generative hierarchy:
\[
G \sim P_0(G), 
\quad \theta \sim P_{\rm meta}(\theta),
\quad s = [(x,y) \mid \theta] \sim P_{\mathrm{task}}(s \mid \theta),
\]
where $P_0$ is the meta-prior over model parameters, $P_{\rm meta}$ is the meta-distribution over tasks, and $P_{\mathrm{task}}$ is the task-specific data distribution. During meta-training, the \emph{System} updates its prior $P_0(G)$ to a data-dependent posterior $
Q(G \mid S) \in \Delta_{\mathcal{G}}$, while the User updates its prior $P_{U}(c \mid \Theta)$ to a posterior over communication contexts $
\tilde{Q}(c \mid \Theta) \in \Delta_{\mathcal{C}}$.   

The meta-distribution $P_{\mathrm{meta}}$ and task-level distributions $P_{\mathrm{task}}(\cdot\mid\theta)$ are modelling assumptions and are not assumed to be known to the User.
The distribution $\tilde Q$ denotes an optimisation-induced (PAC--Bayes) posterior over contexts that characterises the solution of a meta-risk minimisation problem under $P_{\mathrm{meta}}$, rather than a belief update based on explicit knowledge of these distributions. 
We do not assume that the User observes the same tasks used to train the System.
Rather, both the User’s observed tasks and the System’s training tasks are assumed to be drawn independently from the same meta-distribution $P_{\mathrm{meta}}$.


Let $\theta \sim P_{\mathrm{meta}}$ denote a task drawn from the meta-distribution, and for each task let
$S_\theta^{(N)}
=
\{(x_i,y_i)\}_{i=1}^N
\sim P_{\mathrm{task}}(\cdot\mid\theta)^N$
denote a finite dataset drawn i.i.d.\ from the task-level data distribution and let $\theta_1,\ldots,\theta_M\sim P_{\rm meta}$ be $M$ meta-training tasks observed by the User.
Given a User context $c \in \mathcal C$, the Prompt Interpreter produces an internal context
$\hat c \sim \pi_R^\varphi(\cdot \mid c, S_\theta^{(N)})$ parameterised by $\varphi$,
and the Base Solver with parameters $G $ induces a conditional distribution $P_G (y\mid x,\hat c)$.

In our formulation, the inductive bias is not encoded in the learned context $\hat{c}$ alone, but in the combined restriction imposed by the context space $\mathcal{C}$, the Prompt Interpreter $\pi_R$, and the System hypothesis class $\mathcal{H}$. The context $\hat{c}$ acts as a control signal that selects among these biases, rather than defining them.
We consider a Bayesian approach where the System updates its beliefs using the posterior distribution $Q:X\times \mathcal{C}\to \Delta( \mathcal{G})$, contained in the set $\boldsymbol{\mathcal{Q}}$. 

There are some special cases to be considered.

The alignment loss between the User and the System is defined as
\begin{equation}
\mathcal L_{\mathrm{align}}
:=
\mathbb E_{\theta\sim P_{\mathrm{meta}},\,c\sim\pi_U(\cdot\mid\theta)}
\big[
\lvert
\mathcal L_{\mathrm{sys}}(\theta;c)-\mathcal L_{\mathrm{usr}}(\theta;c)
\rvert
\big],
\label{eq:alignment_loss}
\end{equation}
which measures the expected deviation between the System’s realised behaviour and the User’s intended objective. The Prompt Interpreter does not have direct access to $\theta$ at test time and cannot select internal contexts by oracle minimisation.
Instead, it is trained by minimising the System’s empirical loss averaged over tasks, contexts, and finite datasets:
\begin{equation}
\varphi^\star
\in
\arg\min_{\varphi}\;
\mathbb E_{\theta\sim P_{\mathrm{meta}}}\,
\mathbb E_{c\sim\pi_U(\cdot\mid\theta)}\,
\mathbb E_{S_\theta^{(N)}\sim P_{\mathrm{task}}(\cdot\mid\theta)^N}
\big[
\widehat{\mathcal L}_{\mathrm{sys}}^{(N)}(\theta;c,G ,\varphi)
\big].
\label{eq:router_meta_optimisation}
\end{equation}

This optimisation induces a rewriting mechanism that approximates, in expectation, the best task-relevant internal context that is compatible with the System’s objective and training constraints.

Since, the User is learning or selecting a communication policy across tasks, the User posterior $\tilde{Q}$ is an optimisation-induced distribution that reflects uncertainty over which policy induces minimal expected loss when processed through the LLM System’s interpretive mechanism. The User prior $P_U$ is a fixed reference distribution over prompting strategies that encodes architectural or linguistic inductive biases rather than epistemic beliefs. Therefore, \(P_{U}\) and \(\tilde{Q}\) play the roles of prior and PAC–Bayes posterior over communication strategies in the User (meta-)space,
while \(P_0\) and \(Q\) are the System prior and posterior over Base Solver parameters \(G\).  The User does not update beliefs about $\theta$, which is observed directly.
Instead, the User learns or selects a distribution over communication contexts that performs well across tasks.
We model this via an optimisation-induced posterior $\tilde Q\in\Delta(\mathcal C)$ obtained by minimising a meta-level risk regularised by deviation from a reference measure $P_U$:
\begin{equation}
\tilde Q^\star
\in
\arg\min_{\tilde Q\in\Delta(\mathcal C)}
\left[
\widehat{\mathcal R}_{\mathrm{usr}}^{(M,N)}(\tilde Q)
+
\frac{1}{\lambda M}
D_{\rm KL}(\tilde Q\Vert P_U)
\right].
\label{eq:user_meta_optimisation}
\end{equation}
The empirical losses represent aggregated experience over interactions, not explicit access to labelled task datasets.
Therefore, the distribution $\tilde Q$ should be interpreted as a Gibbs or PAC–Bayes posterior over communication strategies rather than as an epistemic belief over tasks. The regularisation term $D_{\rm KL}(\tilde Q\Vert P_U)$ models inertia or bias toward familiar contexts induced by repeated interaction, and is not intended as a literal cognitive computation performed by the User.

The User’s empirical meta-risk is the average, over observed tasks, of the empirical User loss induced by contexts sampled from $ \tilde Q$ is:
\begin{equation}
\widehat{\mathcal R}_{\mathrm{usr}}^{(M,N)}(\tilde Q)
\;:=\;
\frac{1}{M}
\sum_{i=1}^M
\mathbb E_{c\sim \tilde Q}
\Big[
\widehat{\mathcal L}_{\mathrm{usr}}^{(N)}(\theta_i;c \mid S_i)
\Big],
\label{eq:emp_user_meta_risk}
\end{equation}

Define the meta-risk under posteriors $(Q,\tilde Q)$ as
\begin{equation}
\mathcal R_{\mathrm{sys}}(Q,\tilde Q)
:=
\mathbb E_{\theta\sim P_{\mathrm{meta}}}
\mathbb E_{c\sim \tilde Q}
\mathbb E_{G\sim Q}
\big[
\mathcal L_{\mathrm{sys}}(G , \varphi;\theta,c)
\big],
\label{eq:meta_risk_skeleton}
\end{equation}
where the empirical User meta-risk $\widehat{\mathcal R}_{\mathrm{usr}}^{(M,N)}(\tilde Q)$ is defined as the average, over observed tasks, of the empirical User loss induced by contexts sampled from $\tilde Q$, as given in~\eqref{eq:emp_user_meta_risk}. The empirical counterpart of~\eqref{eq:meta_risk_skeleton} over $M$ meta-training tasks is
\begin{equation}
\widehat{\mathcal R}_{\mathrm{sys}}^{(M,N)}(Q,\tilde Q)
:=
\frac{1}{M}\sum_{i=1}^M
\mathbb E_{c\sim \tilde Q}
\mathbb E_{G\sim Q}
\big[
\widehat{\mathcal L}^{(N)}_{\mathrm{sys}}(G , \varphi;\theta_i,c)
\big].
\label{eq:emp_meta_risk_skeleton}
\end{equation}
The User population risk captures the best achievable task performance from the User’s perspective, independent of the internal objectives optimised by the System.

For each task $\theta$ and context distribution $\tilde Q\in\Delta(\mathcal C)$ define the population and empirical \emph{task-level} User risks:
\begin{align}
\mathcal R_{\mathrm{usr}}(\theta;\tilde Q)
:=
\mathbb E_{c\sim\tilde Q}\big[\mathcal L_{\mathrm{usr}}(\theta;c)\big],\quad
\widehat{\mathcal R}_{\mathrm{usr}}^{(N)}(\theta;\tilde Q\mid S_\theta^{(N)})
:=
\mathbb E_{c\sim\tilde Q}\big[\widehat{\mathcal L}_{\mathrm{usr}}^{(N)}(\theta;c\mid T_\theta^{(N)},S_\theta)\big].
\end{align}
\textcolor{black}{
The empirical User risk $\widehat{\mathcal R}_{\mathrm{usr}}^{(N)}$ is introduced solely as a statistical estimator of the population User risk $R_{\mathrm{usr}}$, and does not represent a quantity observed or optimised by the User during interaction. It measures how closely the System’s prompt-conditioned predictive distribution approximates the User-ideal task behaviour when evaluated on $N$ independent samples from the task distribution.} Define the User meta-risk under $\tilde Q\in\Delta(\mathcal C)$ as
\begin{equation}
R_{\mathrm{usr}}(\tilde Q)
:=
\mathbb E_{\theta\sim P_{\mathrm{meta}}}\;
\mathbb E_{c\sim \tilde Q}\Big[\mathcal L_{\mathrm{usr}}(\theta;c)\Big],
\label{eq:user_meta_risk}
\end{equation}
where $\mathcal L_{\mathrm{usr}}(\theta;c)$ denotes the User’s ideal task loss under context $c$.

We can then consider the regularised objectives
\begin{equation}
J_{\mathrm{sys}}(Q;\tilde Q)
:=
\widehat{\mathcal R}_{\mathrm{sys}}^{(M,N)}(Q,\tilde Q)
+\frac{1}{\beta M}D_{\rm KL}(Q\Vert P_0),
\;
J_{\mathrm{usr}}(\tilde Q)
:=
\widehat{\mathcal R}_{\mathrm{usr}}^{(M,N)}(\tilde Q)
+\frac{1}{\lambda M}D_{\rm KL}(\tilde Q\Vert P_U),
\label{eq:regularised_objectives_skeleton}
\end{equation}
where $\widehat{\mathcal R}_{\mathrm{usr}}^{(M,N)}$ is the User's empirical meta-objective used in our analysis (e.g., PAC–Bayes/regularized optimisation) as an empirically measurable proxy for the population User risk
(defined analogously to \eqref{eq:emp_meta_risk_skeleton} using the User loss),
and $\beta,\lambda>0$ are complexity parameters. The setup can be formalised as a \textit{bilevel optimisation program} \cite{colson2007overview}:
\begin{align}
\label{eq:bilevel_program}
\inf_{Q,\tilde Q}J_{\mathrm{sys}}(Q;\tilde Q),\;\;
\text{s.t.}\;
\tilde{Q}^\ast \in \arg\!\min_{\tilde{Q} \in \Delta(\mathcal C)}J_{\mathrm{usr}}(\tilde Q)&,
\end{align}
so that $\tilde{Q}^\ast$ is the minimiser of a PAC–Bayes objective i.e., the risk plus a Kullback-Leibler divergence regulariser. 

\section{Cheap-Talk Analogy and Alignment Loss}
\label{sec:cheap_talk_analogy}

The interaction between the User and the Prompt Interpreter can be interpreted through the lens of \emph{cheap-talk} models \citep{crawford1982strategic, farrell1996cheap}, which we use as an analytical analogy rather than as a fully specified equilibrium game.
In this perspective, the User observes a latent task $\theta \in \Theta$ and communicates via a costless, non-binding message $c \in \mathcal{C}$, while the Prompt Interpreter maps this message to an internal context $\hat{c} \sim \pi_{R}(\cdot \mid c,S)$ that conditions the downstream Base Solver. Because the message $c$ is non-binding, the Prompt Interpreter may reinterpret or transform it in a way that does not faithfully preserve the User’s intent. Objective misalignment captures the case in
which the User’s ideal predictive distribution differs from the System’s, reflecting an analogue of payoff divergence in classical cheap-talk settings, without requiring explicit equilibrium analysis.
Expressivity misalignment, by contrast, is structural: it arises when the communication space $\mathcal{C}$ or the Prompt Interpreter’s mapping lacks sufficient expressivity to encode all
task-relevant distinctions. Our analysis shows that when either form of misalignment is present, a strictly positive KL
separation
\[
D_{\rm KL}\!\left(P_G^{c} \,\|\, P_G^{\theta}\right) > 0
\]
is unavoidable. This separation appears as an irreducible asymptotic error term in the PAC--Bayes generalisation bounds derived in Section~\ref{sec:generalisation}.
If the Prompt Interpreter perfectly preserves task-relevant information so that
$P_{G}(y \mid x, \hat{c}) = P_{G^*}(y \mid x, \theta)$, the alignment loss vanishes.
Otherwise, a positive alignment loss persists, with magnitude determined by the combined effects of
objective mismatch and communication constraints.

While cheap-talk equilibria provide useful intuition for these phenomena, our results do not depend on explicit best-response or equilibrium assumptions.
Instead, the analogy serves to clarify how non-binding communication and misaligned objectives naturally give rise to partial revelation and persistent error in prompt-conditioned learning systems. From an information-theoretic standpoint, the alignment loss can be viewed as increasing with the  information gap between the latent task $\theta$ and the effective context $\hat{c}$: reducing how much the context reveals about the task necessarily increases the expected misalignment penalty.

This perspective also sheds light on alignment phenomena observed in reinforcement learning from human feedback (RLHF). In this setting, the human plays the role of a User transmitting evaluative signals, while the reward model functions as a Prompt Interpreter that interprets these signals according to its own inductive biases.
When objectives diverge, the system may optimise a proxy reward rather than the true human intent. Safety training and content filtering can therefore be viewed as constraints on the Prompt Interpreter, shaping the admissible set of internal interpretations and influencing the resulting
alignment loss.

\section{Alignment Analysis}
In this section, we identify two independent sources of irreducible asymptotic risk. Limited information or representational capacity of the prompt channel induces an expressivity floor, while admissibility or objective constraints on the System induce an objective floor. Together, these results characterise when prompt-based LLMs can generalise from limited data and when unavoidable error persists regardless of sample size.


Before stating the conditions for full revelations, we first introduce some important concepts. Let \(\Theta\sim P_{\mathrm{meta}}\) and let
\(C\mid\Theta=\theta\sim\Pi(\cdot\mid\theta)\), where
\(\Pi:\Theta\to\Delta(\mathcal C)\) is a task-conditional prompting kernel.
We reserve \(\tilde Q\in\Delta(\mathcal C)\) for the PAC--Bayes posterior over
contexts used in Section~\ref{sec:generalisation}.
Denote the induced joint distribution by
\[
P_{\Theta,C}(\mathrm d\theta,\mathrm dc)
:=
P_{\mathrm{meta}}(\mathrm d\theta)\,\tilde Q(\mathrm dc\mid \theta),
\]
and the marginal of \(C\) by
\[
P_C(\mathrm dc)
:=
\int_{\Theta} P_{\mathrm{meta}}(\mathrm d\theta)\,\tilde Q(\mathrm dc\mid \theta).
\]
The mutual information between \(\Theta\) and \(C\) is defined by
\begin{equation}
I(\Theta;C)
:=
D_{\rm KL}\!\big(P_{\Theta,C}\,\big\|\,P_{\mathrm{meta}}\otimes P_C\big)
=
\mathbb E_{\theta\sim P_{\mathrm{meta}}}
\Big[
D_{\rm KL}\!\big(\Pi(\cdot\mid\theta)\,\|\,P_C\big)
\Big].
\label{eq:mutual_info_def}
\end{equation}
The quantity \(I_{\Pi}(\Theta;C)\) measures how statistically dependent the prompt is on the latent task under the task-conditional prompting kernel \(\Pi\). Equivalently, it quantifies how much uncertainty about $\Theta$ is reduced by observing $C$. If $I_{\Pi}(\Theta;C)=0$, then prompts are independent of tasks and no task inference is possible; if $I_{\Pi}(\Theta;C)=H(\Theta)$, then the prompt fully reveals the task. Intermediate values correspond to partial revelation: distinct tasks may induce overlapping prompt distributions, leading to task aliasing. The mutual information therefore captures the effective task information that survives the linguistic encoding and any rewriting or policy constraints, and it is precisely this quantity that governs whether the prompt channel can resolve the entropy of the task family.
We now introduce a key concept, namely the information gap. Let \(\mathcal A\) denote the admissible class of User prompting policies
(e.g.\ those induced by a KL--regularised objective).
Define the information capacity of the channel class as
\begin{equation}
I(\Theta;C)_{\max}
:=
\sup_{\Pi\in\mathcal A} I_{\Pi}(\Theta;C),
\label{eq:I_max_def}
\end{equation}
where \(I_{\Pi}(\Theta;C)\) emphasises that the joint law, and hence the mutual information, depends on the prompting kernel \(\Pi\). The information gap of a particular policy \(\Pi\) is
\begin{equation}
\Delta I(\Pi)
:=
I(\Theta;C)_{\max}-I_{\Pi}(\Theta;C).
\label{eq:Delta_I_def}
\end{equation}

When the User objective is regularised, optimisation restricts attention to a subset of policies with bounded information,
for example $\mathcal A(B):=\{\Pi\in\mathcal A: I_{\Pi}(\Theta;C)\le B\}$ for some finite $B$.
All information-theoretic lower bounds below may be interpreted as holding uniformly over such an admissible class.

We distinguish the task-conditional prompting kernel
\(\Pi:\Theta\to\Delta(\mathcal C)\), used in the information-theoretic
analysis of task revelation, from the PAC--Bayes posterior
\(\tilde Q\in\Delta(\mathcal C)\), used in the empirical User-risk bounds.
Thus \(C\mid\Theta=\theta\sim \Pi(\cdot\mid\theta)\) in Theorem~\ref{thm:blended_floor_usr},
whereas \(c\sim\tilde Q\) in Theorems~\ref{lem:task_pb}--\ref{cor:user_negative}.

For a task-conditional prompting kernel
\(\Pi:\Theta\to\Delta(\mathcal C)\), define
\[
R_{\mathrm{usr}}(\Pi)
:=
\mathbb E_{\theta\sim P_{\mathrm{meta}}}
\mathbb E_{c\sim\Pi(\cdot\mid\theta)}
\big[
\mathcal L_{\mathrm{usr}}(\theta;c)
\big].
\]
This is distinct from the PAC--Bayes User risk
\(R_{\mathrm{usr}}(\tilde Q)\), where \(\tilde Q\in\Delta(\mathcal C)\)
is a task-independent posterior over contexts.

\paragraph{Fully revealing communication.}
We first isolate an idealised regime in which the User’s context can disambiguate the latent task for the Prompt Interpreter,
so that the System’s behaviour can be conditioned on $\theta$ without distortion from the communication channel.

In this setting, the User's context $c$ allows the Prompt Interpreter to fully
reveal the true underlying task $\theta$ to the Base Solver.  
As a sufficient condition for full revelation, suppose there exists a
measurable map $f:\Theta \to \mathcal{C}$ such that $\pi_{U}(c \mid \theta) = \delta_{f(\theta)}(c)$ for all
$\theta$, where $\delta_{\theta_0}:\Theta\to\mathbb{R}_{>0}$ is the Dirac-delta function is a generalised function (or distribution) that satisfies \(
\mathbb{E}_{\delta_{\theta_0}(\theta)}\left[f(\theta)\right]=\int\delta_{\theta_0}(\theta)f(\theta)d\theta=f(\theta_0)\) and $f$ is injective with measurable inverse $g:\mathcal{C}\to\Theta$ satisfying
$g(f(\theta)) = \theta$. Hence, from the Prompt Interpreter's perspective, observing $c$ determines $\theta$ via
$\theta = g(c)$.

Therefore, in a \emph{fully revealing} regime we can consider deterministic User policies of the form
\[
\pi_U(c\mid\theta)=\delta_{f(\theta)}(c),
\]
where \(f:\Theta\to\mathcal C\) is injective and admits a measurable inverse. 

This construction ensures that the latent task \(\theta\) can be uniquely recovered from the observed context \(c\), and serves as a convenient sufficient condition for perfect task communication. 

We therefore adopt the following operational definition of full revelation, stated in terms of the System’s induced task inference.
Determinism and injectivity are however stronger than necessary. More generally, the User--System interaction is \emph{fully revealing} if the Prompt Interpreter can infer the task without ambiguity from any context that the User might emit. Formally, full revelation holds if, for every task \(\theta\) and every context \(c\) in the support of the User policy,
\begin{equation}
q_{\varphi}(\theta'\mid c,S)=\delta_{\theta}(\theta')
\qquad
\text{for all } c\in\mathrm{supp}(\pi_U(\cdot\mid\theta)),
\label{eq:full_revelation_general}
\end{equation}
where \(q_{\varphi}(\cdot\mid c,S)\) denotes the task--inference distribution induced by the Prompt Interpreter. This definition requires only identifiability at the level of the System’s inference, and does not impose determinism, injectivity, or uniqueness of the User’s context selection. In particular, the User policy \(\pi_U(\cdot\mid\theta)\) may be stochastic and may assign positive probability to multiple distinct contexts, provided that all such contexts induce the same task inference by the System. Under this condition, the Prompt Interpreter achieves perfect task identifiability, and the System’s effective predictive distribution satisfies
\[
P_{G,\varphi}(y\mid x,c,S)=P_G(y\mid x,\theta),
\]
so that the Base Solver behaves as if the latent task were directly observed. The resulting joint observable distribution over \((x,y,c)\) can therefore be written as
\[
\bar P(x,y,c)
=
\int_{\Theta}
P_{\mathrm{task}}(x,y\mid\theta)\,
\pi_U(c\mid\theta)\,
P_{\mathrm{meta}}(d\theta),
\]
with the property that all contexts \(c\) in the support of \(\pi_U(\cdot\mid\theta)\) induce identical task-conditioned behaviour. In this regime, conditioning on the context \(c\) is equivalent to conditioning on the task \(\theta\), and the PAC--Bayes meta-objective reduces to its standard form involving a posterior over Base Solver hypotheses alone. When the full-revelation condition fails, no admissible context can induce the correct task-conditioned predictive distribution, and a strictly positive irreducible error persists even in the asymptotic regime.

The definition \eqref{eq:full_revelation_general} isolates when the task is identifiable from what the System actually uses.
To separate distinct failure modes, we now decompose “reliable prompt-based task inference” into four orthogonal conditions.

\begin{condition}[Objective alignment]
\label{cond:objective_alignment}
For all tasks \(\theta\in\Theta\) and inputs \(x\),
the User-ideal and System-ideal predictive distributions coincide: $
P^{\mathrm{usr}}_\theta(\cdot\mid x)
=
P^{\mathrm{sys}}_\theta(\cdot\mid x)$.
\end{condition}

\begin{condition}[Task representability]
\label{cond:task_representability}
There exists a measurable mapping \(c:\Theta\to\mathcal C\) such that
for all \(\theta\in\Theta\),
$
D_{\rm KL}\!\left(
P^{\mathrm{sys}}_\theta
\;\big\|\;
P^{\mathrm{sys}}_{c(\theta)}
\right)=0$,
where \(P^{\mathrm{sys}}_{c(\theta)}\) denotes the predictive distribution induced
by the prompt \(c(\theta)\) (including any rewriting).
\end{condition}
\begin{condition}[Task identifiability]
\label{cond:task_identifiability}
There exists a User policy \(\tilde Q^\star\) such that,
with \(C\mid\Theta=\theta\sim\tilde Q^\star(\cdot\mid\theta)\), $
I_{\tilde Q^\star}(\Theta;C)=H(\Theta)$, assuming $H(\Theta)<\infty$. Equivalently, \(\Theta\) is a measurable function of \(C\) almost surely.
\end{condition}

\begin{condition}[Information preservation under rewriting]
\label{cond:preservation}
Let $Z$ denote the effective representation available to the Base Solver after
User prompting and System-side rewriting. Full information preservation holds if
\[
I(\Theta;Z)=I(\Theta;C),
\]
and an information bottleneck is present whenever
\[
I(\Theta;Z)<I(\Theta;C).
\]
When the rewriting mechanism depends on auxiliary information $S$, the relevant
information object is $Z=(\hat C,S)$ if $S$ is available to the solver, and $Z=\hat C$
otherwise. All lower bounds below are stated in terms of the effective task information
$I(\Theta;Z)$ rather than the raw prompt information $I(\Theta;C)$.
\end{condition}

Conditions~\ref{cond:task_identifiability} and ~\ref{cond:preservation} together ensure full revelation in the sense of task identifiability at the System level. Condition ~\ref{cond:task_representability} strengthens this to predictive revelation, guaranteeing that an identifiable task can be correctly realised by the System. Condition~\ref{cond:objective_alignment} further ensures that the realised behaviour coincides with the User’s intended objective. Failure of any condition induces a corresponding irreducible population-level error floor.

Together, these conditions isolate distinct failure modes in prompt-based interaction, corresponding to objective mismatch, representational limitations, information bottlenecks, and information loss due to internal processing. These conditions correspond to increasingly strong notions of revelation familiar from cheap--talk models. 

Relaxing any one of these conditions leads to weaker forms of revelation and induces a corresponding irreducible error floor.
We have therefore,  far considered the case when the User's hypothesis space contains the correct specification of its communication protocol and there is no objective misalignment. This setup allows the underlying task to be perfectly inferred by the System. In practice although one does not have access to the expected loss, we can readily see that minimising the expected loss is a good strategy for minimising the empirical loss. 

The remainder of this section studies the complementary regime in which one or more of these conditions fails,
and shows that each failure induces a strictly positive population-level error floor.

\paragraph{Misaligned or Non-Informative User Context.}
This is case where either the Prompt Interpreter refuses to communicate the inferred intent due to misaligned objectives (for example, when the User is asking for something contrary to safety protocols) or when the User's prompt is not fully revealing.

When the User's prompt $c$ provides only partial information about the latent task $\theta$, 
the conditional distribution $\pi_{U}(c\mid\theta)$ is no longer a Dirac delta but a general stochastic policy.  
The System's rewriting mechanism must therefore form predictions by marginalising over latent tasks consistent with $c$.  The population risk for a fixed prompt–rewriting mechanism $\varphi$ is
\begin{align}
L(\varphi)
&=
\int_{\mathcal{X}\times\mathcal{Y}\times\mathcal{C}}
\ell\!\left(y,\,P_{\varphi}(\cdot\mid x,c,S)\right)
\,\bar P(x,y,c)
\,dx\,dy\,dc.
\label{eq:Lphi_partial1}
\end{align}
By interchanging the order of integration and collecting terms, 
the expected loss can be expressed as $
L(\varphi)
=
\int_{\Theta}
P_{\rm task}(d\theta)
\int_{\mathcal{C}}
\pi_{U}(c\mid\theta)
\int_{\mathcal{X}\times\mathcal{Y}}
 P_{\mathrm{task}}(s,y\mid\theta)\,
\ell\!\left(
y,\,P_{\varphi}(\cdot \mid x,c,S)
\right)
dx\,dy\,dc$.
The inner two integrals define the expected loss for task $\theta$ given its observable context distribution:
\begin{align}
L(\varphi;\theta)
:=
\mathbb{E}_{c\sim\pi_{U}(\cdot\mid\theta)}
\mathbb{E}_{(x,y)\sim P_{\mathrm{task}}(\cdot\mid\theta)}
\!\left[
\ell\!\left(y,\,P_{\varphi}(\cdot\mid x,c,S)\right)
\right],
\end{align}
so that
\begin{align}
L(\varphi)
=
\mathbb{E}_{\theta\sim P_{\rm meta}}
\!\left[L(\varphi;\theta)\right].
\label{eq:Lphi_partial3}
\end{align}

Equations~\eqref{eq:Lphi_partial1} - \eqref{eq:Lphi_partial3} capture the latent–task structure: 
the User's context distribution $\pi_{U}(c\mid\theta)$ only partially encodes $\theta$, 
and the System's predictive model must implicitly infer the task from the observed prompt $c$. 
In this inconsistent setting the System cannot invert the User’s prompt to 
recover $\theta$, and a persistent generalisation floor appears whenever the 
effective communication channel $\pi_R\circ \pi_U$ is non-injective or 
insufficiently expressive. 

Let $q_{\varphi}(\theta\mid c,S)$ denote the implicit posterior over tasks induced by the mechanism $\varphi$. 
The predictive distribution can then be rewritten as a task–marginal mixture:
\begin{align}
P_{\varphi}(y\mid x,c,S)
=
\mathbb{E}_{\theta\sim q_{\varphi}(\cdot\mid c,S)}
\mathbb{E}_{\hat c \sim \pi_R(\cdot\mid c,S)}
\!\left[P_{G}(y\mid x,\hat c)\right].
\end{align}
Under full revelation, conditioning on the effective context is equivalent to conditioning on the latent task,
so the System can in principle realise task-conditioned behaviour without an information bottleneck.
When full revelation fails, distinct tasks remain confounded under the induced channel, which leads to an irreducible population risk floor.

\section{Population-Level Misalignment Bounds}
\label{sec:misalign}  
In this section, we relax Conditions~\ref{cond:objective_alignment} - \ref{cond:preservation} and study four distinct failure modes: (i) objective misalignment between the System and the User, (ii) limited expressivity of the prompt space (so some tasks cannot be induced by any admissible prompt), (iii) insufficient task information transmitted by prompts (so tasks are not identifiable), and (iv) loss of task-relevant information due to System-side rewriting. We establish population-level lower bounds on the User risk that arise from these mechanisms alone. These bounds are independent of finite-sample PAC–Bayes effects.

Our first theorem is proven with the help of the following assumption:
\begin{assumption}[Task separation under shared context]\label{ass:A2task_separation}
There exists \(\delta>0\) such that for all \(i\neq j\),
\begin{equation}
\inf_{(G,\varphi)\in\mathcal H}
\frac12
\Big(
\mathbb E_x D_{\rm KL}(P_{\theta_i}(\cdot\mid x)\,\|\,P_{G,\varphi}(\cdot\mid x,\hat c))
+
\mathbb E_x D_{\rm KL}(P_{\theta_j}(\cdot\mid x)\,\|\,P_{G,\varphi}(\cdot\mid x,\hat c))
\Big)
\;\ge\;
\delta,
\label{eq:task_separation}
\end{equation}
uniformly in \(\hat c\).
\end{assumption}
The assumption requires that
no single rewritten prompt can simultaneously solve two genuinely different tasks.
This excludes degenerate task families and is the weakest condition under which
communication matters.

\begin{theorem}[Irreducible User floor under expressivity misalignment]
\label{thm:blended_floor_usr}
Let \(\Theta_0=\{\theta_1,\ldots,\theta_K\}\subseteq\Theta\) be a finite task
packing with \(P_{\mathrm{meta}}(\Theta_0)=\alpha>0\), and suppose that
conditional on \(\Theta\in\Theta_0\), the task is uniformly distributed over
\(\Theta_0\). Assume that Assumption~\ref{ass:A2task_separation} holds on
\(\Theta_0\), i.e. the tasks remain separated by at least
\(\delta>0\) uniformly over all effective representations
\(\hat c\) available to the Base Solver.
Let \(Z\) denote the effective representation after prompting and rewriting.
If an admissible User--System interface \(\Pi\in\mathcal A(B)\) satisfies
\[
I_\Pi(\Theta;Z\mid \Theta\in\Theta_0)\le B,
\]
then
\[
\liminf_{N\to\infty}R_{\mathrm{usr}}(\Pi)
\ge
\alpha\delta
\left[
1-\frac{B+\log 2}{\log K}
\right]_+:=Y_{\mathrm{expr}}.
\]
In particular, if \(B<\log K-\log 2\), then the expressivity floor is strictly positive.
\end{theorem}
\begin{sproof}
\textit{Fix the separated tasks $\{\theta_1,\dots,\theta_K\}$ from Assumption~\ref{ass:A2task_separation}, with $P_{\mathrm{meta}}(\theta_j)\ge \alpha$ and
\(
\min_{i\neq j}\E_x\Big[\KL\big(P_{\theta_i}(\cdot\mid x)\,\|\,P_{\theta_j}(\cdot\mid x)\big)\Big]\ge \delta .
\)
\\\indent
Step 1 (reduce risk to a task identification problem).
Let $\Theta$ denote a random task supported on $\{\theta_1,\dots,\theta_K\}$ and let \(C\sim\Pi(\cdot\mid\Theta)\) be the (random) prompt/context generated under $\tilde Q$.
Given $(x,C)$, the downstream system must induce some predictive distribution $P(\cdot\mid x,C)$.
By the separation condition, any single predictor cannot match \emph{all} $P_{\theta_j}(\cdot\mid x)$ simultaneously: if the system effectively ``uses the wrong task'' on an event $\{\widehat\Theta(C)\neq \Theta\}$, then it incurs at least $\delta$ expected KL mismatch on that event.
Formally, one can define a decoder $\widehat\Theta:\mathcal C\to\{1,\dots,K\}$ (e.g.\ maximum-likelihood under the induced predictor family) such that
\(
R_{\mathrm{usr}}(\Pi)
\;\ge\;
\alpha\,\delta\;\Pr\!\big(\widehat\Theta(C)\neq \Theta\big)
\), where $\alpha$ enters because each $\theta_j$ has mass at least $\alpha$ in $P_{\mathrm{meta}}$ on this restricted set.
\\\indent Step 2 (control task confusion probability by mutual information).
Since $I_{\Pi}(\Theta;C)\le B$, Fano’s inequality yields a lower bound on the error probability of \emph{any} decoder $\widehat\Theta(C)$:
\(
\Pr\!\big(\widehat\Theta(C)\neq \Theta\big)
\;\ge\;
[1-\frac{I_{\Pi}(\Theta;C)+\log 2}{\log K}]_+
\;\ge\;
[1-\frac{B+\log 2}{\log K}]_+\).
\\\indent Step 3 (combine and pass to $N\to\infty$).
Combining the above inequalities gives the claimed strictly positive lower bound on the population risk, and taking $\liminf_{N\to\infty}$ removes estimation terms (the bound depends only on the interface constraint $B$ and the task separation $\delta$). This gives the result. In particular, if $B<\log K-\log 2$ then the bracketed term is positive, so the expressivity floor is non-vanishing.}
\end{sproof}
\begin{tcolorbox}[
  breakable,
  colback=gray!5,
  colframe=gray!40!black,
  title=\textbf{Interpretation of Theorem~\ref{thm:blended_floor_usr}}
]

The lower bound in Theorem~\ref{thm:blended_floor_usr} is governed by the ratio
\(
\frac{B+\log 2}{\log K},
\)
which compares the effective information capacity of the prompt channel to the intrinsic entropy of the task family. The quantity $B$ upper bounds the mutual information $I_{\Pi}(\Theta;C)$ between latent tasks and prompts under the admissible class of User policies. It therefore measures how many bits about the task can survive the entire User--Interpreter pipeline after accounting for linguistic ambiguity, rewriting, safety filtering, and regularisation. Crucially, $B$ captures \emph{effective} task information, not raw context length. The denominator $\log K$ approximates the entropy of a $K$-task family under a near-uniform prior and represents the number of bits required to uniquely identify one task among $K$ predictively distinguishable alternatives. The ratio therefore,  quantifies the fraction of task entropy that the prompt interface can actually transmit.

When $B \ge \log K - \log 2$, the bracketed term vanishes and the information bottleneck does not force a population-level floor. However, if $B < \log K - \log 2$, the prompt channel cannot disambiguate all tasks: task aliasing occurs with strictly positive probability. In this regime the User population risk satisfies
\[
\liminf_{N\to\infty} R_{\mathrm{usr}}(\tilde Q) \;\ge\; Y_{\mathrm{expr}} \;>\; 0,
\]
so the learning curve exhibits an \emph{asymptotic floor}. This floor reflects expressivity or identifiability misalignment: the communication space $\mathcal{C}$ does not contain contexts that fully reveal $\theta$, and hence $D_{\rm KL}(P_G^{\pi_U(c|\theta)} \Vert P_G^{\theta})>0$ on a non-negligible set of tasks.

From the perspective of cheap--talk \citep{crawford1982strategic,farrell1996cheap}, this corresponds to partial revelation: the prompt partitions the task space into coarse equivalence classes rather than uniquely identifying the latent task. Even if the downstream solver were Bayes-optimal conditional on $c$, it would still incur distortion within each equivalence class. No amount of additional data can eliminate this distortion, because the ambiguity is upstream of learning. The limitation is structural, not statistical.

Importantly, increasing context length does not automatically remove the floor. While longer contexts may increase the representational capacity of $\mathcal{C}$, the relevant quantity in the theorem is the achievable mutual information $I(\Theta;C)$ under admissible prompting and rewriting mechanisms. If safety constraints, stylistic regularisation, or linguistic structure limit this mutual information, then $B$ remains bounded even as raw context length grows. The critical scaling question is therefore not whether the model is larger, but whether the interface can transmit sufficient information to resolve the entropy of the task family. This clarifies why instruction-tuned or safety-aligned models can exhibit persistent performance gaps under ambiguous or over-constrained prompts. Reducing $Y_{\mathrm{expr}}$ requires increasing effective task information for example, by enriching communication protocols, improving semantic preservation in interpreter layers, or introducing interaction and feedback loops. 
\end{tcolorbox}
The key insight of the result is that the User incurs an additional loss quantified by $Y_{\mathrm{expr}}$. 

We now tackle the case in which we assume Condition~\ref{cond:objective_alignment} (Objective alignment) does not hold so that the System’s notion of correctness is not assumed to match the User’s intended task outcome. To prove our next result, we make the following assumption

\begin{assumption}[Non-emptiness of admissibility set.]\label{ass:A3admissibility_set}
For each input-context pair \((x,c)\), there exists a non-empty set
\(\mathcal P_{\mathrm{safe}}(x,c)\subseteq \Delta(\mathcal Y)\) of admissible predictive distributions.
We assume that the effective System predictor is constrained to be admissible, in the sense that
\begin{equation}
P_{G,\varphi}(\cdot\mid x,c)\ \in\ \mathcal P_{\mathrm{safe}}(x,c)
\qquad \text{for all admissible } (G,\varphi)\in\mathcal H \text{ and all } (x,c).
\label{eq:safe_set_assumption}
\end{equation}
\end{assumption}
Assumption~\ref{ass:A3admissibility_set} formalises the idea that, for any given input and prompt, the System is restricted to a set of admissible predictive behaviours, for example due to safety, alignment, or deployment constraints. The non-emptiness condition simply ensures that these constraints are not contradictory, so that at least one valid prediction is always available. This assumption rules out degenerate settings in which the System is forced to violate its own constraints, and allows us to meaningfully characterise the irreducible error induced by admissibility restrictions rather than by infeasibility.

We now present our second main result:

\begin{theorem}[Objective-misalignment floor for the User meta-risk]
\label{thm:misalignment_flr_meta}
Suppose Assumption~\ref{ass:A3admissibility_set} holds and let \(\mathcal P_{\mathrm{safe}}(x,c)\subseteq\Delta(\mathcal Y)\) be a nonempty admissible set for each \((x,c)\).
Define the objective distortion
$
\Delta_{\mathrm{obj}}(\theta,x,c)
:=
\inf_{Q\in\mathcal P_{\mathrm{safe}}(x,c)}
D_{\rm KL}\!\big(P^{\mathrm{usr}}_\theta(\cdot\mid x)\,\big\|\,Q\big)$.
Then for any task-conditional User policy \(\Pi:\Theta\to\Delta(\mathcal C)\),
\begin{equation}
R_{\mathrm{usr}}(\Pi)
\ \ge\
Y_{\mathrm{obj}}
:=
\mathbb E_{\theta\sim P_{\mathrm{meta}}}
\mathbb E_{c\sim \Pi(\cdot\mid\theta)}
\mathbb E_{x\sim P_{\mathrm{task}}(\cdot\mid\theta)}
\big[\Delta_{\mathrm{obj}}(\theta,x,c)\big].
\label{eq:Rusr_ge_objgap}
\end{equation}
Moreover, if there exist constants \(\varepsilon>0\) and \(\beta\in(0,1]\) such that
\begin{align}
\mathbb P\!\left(\Delta_{\mathrm{obj}}(\Theta,X,C)\ge \varepsilon\right)\ \ge\ \beta,\label{eq:obj_sep_condition}
\end{align}
where \(\Theta\sim P_{\mathrm{meta}}\), \(X\sim P_{\mathrm{task}}(\cdot\mid\Theta)\),
and \(C\mid\Theta\sim \Pi(\cdot\mid\Theta)\),
then 
\begin{align}
R_{\mathrm{usr}}(\Pi)\ge \beta\,\varepsilon.\label{eq:Yobj_positive}\end{align}
\end{theorem}
\begin{sproof}
\textit{Fix any admissible System $(G,\phi)\in\mc H$. By Assumption~\ref{ass:A3admissibility_set}, for every $(x,c)$ we have $P_{G,\phi}(\cdot\mid x,c)\in\mc P_{\mathrm{safe}}(x,c)$, hence by the definition of the distance-to-set distortion and Lemma~3,
\(
D_{\rm KL}\!\big(P^{\mathrm{usr}}_\theta(\cdot\mid x)\,\big\|\,P_{G,\phi}(\cdot\mid x,c)\big)\ \ge\ \inf_{Q\in\mc P_{\mathrm{safe}}(x,c)}D_{\rm KL}\!\big(P^{\mathrm{usr}}_\theta(\cdot\mid x)\,\big\|\,Q\big)=\Delta_{\mathrm{obj}}(\theta,x,c).
\)
Taking expectations over $x\sim P_{\mathrm{task}}(\cdot\mid\theta)$, then $c\sim\tilde Q(\cdot\mid\theta)$, then $\theta\sim P_{\mathrm{meta}}$ yields \eqref{eq:Rusr_ge_objgap}. For \eqref{eq:Yobj_positive}, use the tail lower bound $\E[\Delta_{\mathrm{obj}}]\ge \varepsilon\,\Pr(\Delta_{\mathrm{obj}}\ge \varepsilon)\ge \beta\varepsilon$ and combine with \eqref{eq:Rusr_ge_objgap}.}
\end{sproof}
\begin{tcolorbox}[
  breakable,
  colback=gray!5,
  colframe=gray!40!black,
  title=\textbf{Interpretation of Theorem~\ref{thm:misalignment_flr_meta}}
]
Theorem~\ref{thm:misalignment_flr_meta} isolates a second source of irreducible error that is independent of task identifiability: objective misalignment induced by admissibility constraints. The distortion term $\Delta_{\mathrm{obj}}(\theta,x,c)$ measures the \emph{distance-to-set} between the User-ideal predictive distribution and the System’s admissible class at $(x,c)$. If, on a non-negligible fraction of task–input–context triples, this distance is bounded below by $\varepsilon$, then the expected User meta-risk admits a strictly positive floor $Y_{\mathrm{obj}}\ge \beta\varepsilon$. Importantly, this floor persists even when the task is perfectly identified and optimisation is exact within $\mathcal P_{\mathrm{safe}}$: the limitation arises not from uncertainty, but from structural constraints on permissible outputs. Therefore, whereas Theorem~\ref{thm:blended_floor_usr} captures \emph{information loss} in the prompt channel, Theorem~\ref{thm:misalignment_flr_meta} captures \emph{objective distortion} imposed by alignment or safety policies. Eliminating this floor requires enlarging the admissible set itself, not increasing data or model capacity.
\end{tcolorbox}

The objective-misalignment floor established in Theorem \ref{thm:misalignment_flr_meta} is conceptually distinct from the representation (or expressivity)
misalignment floor derived in Theorem~\ref{thm:blended_floor_usr}.
The former arises when the User-ideal predictive distribution lies outside the admissible set enforced by the System,
so that no admissible predictor can realise the User objective even with perfect task inference and unlimited data.
By contrast, the representation-misalignment floor (Theorem~\ref{thm:blended_floor_usr}) arises when the effective prompt representation carries insufficient
information to disambiguate a finite set of predictively distinct tasks, even when the System objective itself is
well aligned with the User.
Together, these results identify two orthogonal and complementary mechanisms by which prompt-based language models can
exhibit irreducible generalisation error: one driven by admissibility constraints on outputs, and the other by
information constraints on tasks.

\section{Generalisability Analysis}\label{sec:generalisation}
In this section, we provide PAC–Bayesian generalisation guarantees for prompt-conditioned learning, separating estimation error from structural limitations. 

\textcolor{black}{The PAC-Bayes bounds derived in this section provide statistical guarantees on how well a prompt-based interaction can approximate the User’s intended task, separating estimation error from irreducible error due to misalignment or limited expressivity. The bounds control the gap between population risks, which capture fundamental limitations of the User–System interface, and empirical risks, which are finite-sample estimators used for analysis. The complexity terms quantify how specialised the prompting strategy and System hypotheses are relative to fixed priors, while the generalisation gap describes the extent to which empirical performance may deviate from population behaviour at finite sample sizes. Our results in this section show that even when these estimation terms vanish, structural misalignment induces a strictly positive residual error.
} The analysis in this section can be viewed as an extension of the hierarchical PAC-Bayes meta-learning framework of~\cite{rezazadeh2022unified}. The key difference is that tasks are not directly observed by the learner. Instead they must be communicated through a constrained prompt channel and subsequently interpreted through a potentially misaligned rewriting mechanism.

We first derive a finite-sample PAC--Bayes bound on the User meta-risk (Section~\ref{subsec:learning_with_alignment}), showing that empirical performance concentrates around population risk as the number of tasks and samples grows. We then show that this convergence does not guarantee vanishing error.

First, under  the assumption that Conditions~\ref{cond:objective_alignment} - \ref{cond:preservation} hold, we provide a PAC-Bayes generalisation bound for which the Bayesian posterior that minimises the bound also minimises the System log-loss function in the asymptotic sample limit in probability. 
This distribution is characterised by a Dirac-delta distribution whose probability mass is centered around $\mathfrak{c}^\star(\theta)$ for any $\theta\in\Theta$ where is any distribution that satisfies $D_{\rm KL}(P^{\mathfrak{c}^\star(\theta)}_{G}\|P^{\theta}_{G})=0$ for any $G\in\mathcal{G}$ and for any $\theta\in\Theta$.

To ensure the performance
of training loss for the choice of $G$ is small with high probability
as the performance of per-task generalisation loss,  we must bound
generalisation gap averaged over the posterior distribution $Q$. We first define the following gap:
\begin{align}
   \Delta^{(M,N)}_{\rm usr}(Q):= \Userpostexp - \Userpostemp. 
\end{align}
We study the expected User generalisation gap \[
   \mathbb{E}_{C\sim \tilde{\mathcal{Q}}}\left[\Delta^{(M,N)}_{\rm usr}(Q)\right],\] which we define as the \textit{expected User gap}. Our goal is to derive a bound on the expected User gap: a small gap means that the performance of the choice of communication context $c$ on the (meta-)training set indicates a reliable measure of the User loss (in probability). In what follows, we present a bound on the gap between the User's empirical loss 
   and the expected loss (which is unknown), $P_U$ denotes the User prior over contexts, $P_0$ denotes the System prior
over model parameters, $\tilde Q$ is the User posterior over contexts, and $Q_\theta$ is the
System posterior over $G$ conditioned on task $\theta$.

\subsection{PAC-Bayes Bounds with Alignment}\label{subsec:learning_with_alignment}

We now derive our results concerning the learning theory of a general-purpose language model under conditions of full alignment. Throughout this section, the distributions 
$Q$ and  $\tilde{Q}$ denote optimisation-induced (PAC–Bayes) posteriors over hypotheses and contexts respectively, rather than epistemic belief updates.

Under Conditions~\ref{cond:objective_alignment} - \ref{cond:preservation},
the following PAC--Bayes bound characterises generalisation in the fully aligned,
task-sufficient regime. For each task $\theta$, we allow a task-indexed System posterior $Q_\theta\in\Delta(\mathcal G)$.

We begin with the following results that study our PAC-Bayes analysis in the fully revealing case.

\begin{theorem}[Reduction to standard PAC--Bayes meta-learning]
\label{thm:reduction_rezazadeh}
Suppose that the prompt channel is fully revealing and the User and System objectives are aligned. Specifically, assume that
\[
I(\Theta;Z)=H(\Theta),
\]
so that the latent task \(\Theta\) is measurable with respect to the observable context variable \(Z\), and that for every task \(\theta\in\Theta\),
\[
P^{\mathrm{usr}}_{\theta}(\cdot\mid x)
=
P^{\mathrm{sys}}_{\theta}(\cdot\mid x)
\qquad
\text{for all }x\in\mathcal X .
\]
Then the GPAI framework reduces to the standard PAC--Bayes meta-learning setting e.g.~\cite{rezazadeh2022unified}. In particular,
\[
Y_{\mathrm{expr}}=Y_{\mathrm{obj}}=0,
\]
and the User meta-risk satisfies the usual PAC--Bayes meta-learning bound with an environment-level complexity term and a task-level complexity term. Equivalently, Theorems~\ref{lem:task_pb}--\ref{cor:user_negative} collapse to the standard estimation-only PAC--Bayes meta-learning regime, with no irreducible communication or objective-misalignment floor.
\end{theorem}

\begin{theorem}[Task-level PAC--Bayes generalisation under unbounded loss]
\label{lem:task_pb} Under Conditions~\ref{cond:objective_alignment} - \ref{cond:preservation}, given $\theta \in \Theta$, for any prior $P_U$ over $\mathcal C$
and any $\delta\in(0,1)$, with probability at least \(1-\delta\) over the independent evaluation sample \(T_\theta^{(N)}\),
\begin{align}
\mathcal R_{\mathrm{usr}}(\theta;\tilde Q)
\le
\widehat{\mathcal R}_{\mathrm{usr}}^{(N)}
(\theta;\tilde Q\mid T_\theta^{(N)},S_\theta)
+
G_{\mathrm{Task}}^{\mathrm{usr}}
\left(
\frac{
D_{\mathrm{KL}}(\tilde Q\Vert P_U)
+
\log(C_{\mathrm{Task}}/\delta)
}{N}
\right).    
 \label{eq:user_task_pb}
\end{align}

simultaneously for all posteriors \(\tilde Q\in\Delta(\mathcal C)\) where \(G_{\mathrm{Task}}^{\mathrm{usr}}\) is the calibration function
from Assumption~\ref{ass:usr_affine_transform}. In particular, if
\(F_{\mathrm{Task}}^{\mathrm{usr}}(a,b)=\lambda_t(a-b)^2\), then
\(
G_{\mathrm{Task}}^{\mathrm{usr}}(u)
=
\sqrt{\frac{u}{\lambda_t}}.
\)
\end{theorem}

\begin{sproof}
\textit{We first introduce a bivariate convex function $F$ to connect the empirical loss and expected loss and then perform a change of measure using the Donsker–Varadhan variational formula to derive a bound between the posterior  and prior distributions. Then, we apply the Markov inequality to bound the expectation of $F$ with the logarithm of the confidence parameter $\delta$ and perform an affine transformation to generate the final bound.}     
\end{sproof}

The theorem bounds the task-level generalisation gap for the User. It controls the discrepancy between the population User risk and its empirical estimate for a fixed task \(\theta\).
 The gap differs from standard meta-learning PAC Bayes bounds owing to the different objectives between the User and the System. Moreover, as earlier remarked, proving this result differs from standard proof techniques due to the need to circumvent issues produced by the non-boundedness of the System loss objective.
Following the development of Theorem \ref{thm:user_meta_pb}, we can tighten the bound using an alternative set of concentration bounds (complete details are deferred to the Appendix). 

\begin{tcolorbox}[colback=gray!5,colframe=gray!40!black,title=\textbf{Interpretation of Theorem~\ref{lem:task_pb} (Aligned Generalisation Bound)}]
Theorem~\ref{lem:task_pb} establishes a PAC--Bayes bound on the User's expected generalisation gap under the condition that the communication between the User and the Prompt Interpreter is sufficiently informative for the Prompt Interpreter to infer the underlying task. The bound scales as
$\frac{D_{\rm KL}(\tilde Q\Vert P_U)}{N}$,
showing that the number of samples $N$ contribute to tighter generalisation guarantees. In this regime, the User's empirical loss is a reliable predictor of its expected loss.   In the cheap--talk analogy this corresponds to near truth--telling in which the message $c$ reveals almost all of the private information $\theta$ to the receiver. Effective prompt design or instruction tuning increases the informativeness of $P(c\mid\theta)$ and therefore strengthens alignment between the User and the Prompt Interpreter. Practically, this suggests that improving prompt calibration or expanding the prompt vocabulary can reduce the generalisation gap by enhancing mutual information between $c$ and $\theta$, thereby achieving a regime where few examples suffice for adaptation.
\end{tcolorbox}

\subsection{Generalisation Bounds under Misalignment}\label{sec:misalign_bounds}
In this section, we study the case in which Conditions~\ref{cond:objective_alignment} - \ref{cond:preservation} no longer hold. Throughout this section, we treat the User prompting strategy as a \emph{conditional} policy
\(\tilde Q:\Theta\to\Delta(\mathcal C)\), writing \(C\mid\Theta=\theta\sim \tilde Q(\cdot\mid\theta)\).
For the System-side learning component, we allow a task-indexed posterior \(Q_\theta\) over Base Solver
hypotheses (or parameters), with a fixed prior \(P_0\); when working with \(M\) sampled tasks
\(\{\theta_i\}_{i=1}^M\), we write \(Q_i:=Q_{\theta_i}\).
All KL terms below are assumed finite whenever they appear.

The following result provides conditions when minimising the PAC-Bayes bound is an optimal strategy for obtaining $\mathcal{Q}^\star$:
\begin{lemma}[Pointwise optimality under cross-entropy]
\label{lem:pointwise_optimality}
Let $\mathcal H$ be the admissible solver class and let $Q\in\Delta(\mathcal H)$.
Define
\[
\mathrm{CE}(Q)
:=
\mathbb E_{(x,y)\sim P_{\mathrm{task}}(\cdot\mid\theta)}
\Big[
-\log \mathbb E_{h\sim Q}\, p_h(y\mid x,c,S)
\Big].
\]
Let
\[
h^\star \in \arg\min_{h\in\mathcal H}
\mathbb E_{(x,y)\sim P_{\mathrm{task}}(\cdot\mid\theta)}
\big[
-\log p_h(y\mid x,c,S)
\big].
\]
A distribution $Q$ minimises $\mathrm{CE}(Q)$ if and only if no mixture over $\mathcal H$ attains strictly smaller cross-entropy than $h^\star$.
In this case, an optimal minimiser is the Dirac measure $Q^\star=\delta_{h^\star}$.
\end{lemma}
This result is an adaptation of Lemma 2 in \cite{masegosa2020learning}. The result indicates that when the model class is misspecified, the solver that minimises expected population loss is optimal among all mixture predictors if and only if no mixture can produce a lower cross-entropy loss. In this case, the optimal posterior collapses to a single solver, justifying why irreducible error floors can be characterised by pointwise solvers even though the analysis allows general posterior mixtures.

The following result characterises the PAC-Bayes generalisation bound under the setting in which 
at least one of 
Condition \ref{cond:task_identifiability}
or Condition~\ref{cond:preservation} fails.
\begin{theorem}[Meta-level PAC--Bayes bound with \(M\) tasks and \(N\) samples per task]
\label{thm:user_meta_pb}
Let \(P_U\in\Delta(\mathcal C)\) be a fixed prior over contexts. Under
Assumptions~\ref{ass:task_exp_moment_usr}, \ref{ass:usr_affine_transform},
and \ref{ass:env_subgaussian_usr}, with probability at least \(1-\delta\)
over the draw of \(\{(\theta_i,T_i^{(N)},S_i)\}_{i=1}^M\), for every
posterior \(\tilde Q\in\Delta(\mathcal C)\),
\[
R_{\mathrm{usr}}(\tilde Q)
\le
\widehat{\mathcal R}_{\mathrm{usr}}^{(M,N)}(\tilde Q)
+
G_{\mathrm{Task}}^{\mathrm{usr}}(B_t)
+
G_{\mathrm{Env}}^{\mathrm{usr}}(B_e),
\]
where
\[
B_t
=
\frac{
D_{\rm KL}(\tilde Q\Vert P_U)+\log(2M C_{\mathrm{Task}}/\delta)
}{N},
\qquad
B_e
=
\frac{
D_{\rm KL}(\tilde Q\Vert P_U)+\log(2 C_{\mathrm{Env}}/\delta)
}{M}.
\]
and \(G_{\mathrm{Task}}^{\mathrm{usr}}\) and
\(G_{\mathrm{Env}}^{\mathrm{usr}}\) are the calibration functions from
Assumption~\ref{ass:usr_affine_transform}. Under the quadratic
specialisations
\[
F_{\mathrm{Task}}^{\mathrm{usr}}(a,b)=\lambda_t(a-b)^2,
\qquad
F_{\mathrm{Env}}^{\mathrm{usr}}(a,b)=\lambda_e(a-b)^2,
\]
they reduce to
\[
G_{\mathrm{Task}}^{\mathrm{usr}}(u)
=
\sqrt{\frac{u}{\lambda_t}},
\qquad
G_{\mathrm{Env}}^{\mathrm{usr}}(u)
=
\sqrt{\frac{u}{\lambda_e}}.
\]
\end{theorem}
Theorem~\ref{thm:user_meta_pb} extends the task-level result to the meta-learning setting by additionally controlling the error incurred when generalising across tasks drawn from \(P_{\mathrm{meta}}\). The bound in the theorem separates a finite-sample meta-generalisation term (vanishing with $M,N$)
from any irreducible error sources.
%
%


\begin{assumption}[Asymptotic regime and growth conditions]
\label{ass:growth_conditions}
The User posterior satisfies
\[
\frac{D_{\rm KL}(\tilde Q\Vert P_U)}{N}\to0,
\qquad
\frac{D_{\rm KL}(\tilde Q\Vert P_U)}{M}\to0
\]
as \(M,N\to\infty\). Hence \(B_t\to0\) and \(B_e\to0\).
\end{assumption}
The assumption ensures that the task-level and environment-level PAC--Bayes penalty terms vanish.
Under these conditions, the empirical risk converges to the population risk, while the irreducible floors
derived below remain unaffected. We can now deduce the following result: 
\begin{theorem}[Negative result for User-level general-purpose learning]
\label{cor:user_negative}
Assume the hypotheses of Theorems~\ref{thm:blended_floor_usr} and
\ref{thm:misalignment_flr_meta} hold for the same admissible class, and define
\[
Y:=\max\{Y_{\mathrm{expr}},Y_{\mathrm{obj}}\}.
\]
Under the assumptions of Theorem~\ref{thm:user_meta_pb}, with probability at least
\(1-\delta\),
\[
\widehat{\mathcal R}_{\mathrm{usr}}^{(M,N)}(\tilde Q)
\ge
Y
-
G_{\mathrm{Task}}^{\mathrm{usr}}(B_t)
-
G_{\mathrm{Env}}^{\mathrm{usr}}(B_e).
\]
Moreover,
\[
\liminf_{M,N\to\infty}
\inf_{\tilde Q\in\mathcal A(B)}
R_{\mathrm{usr}}(\tilde Q)
\ge
Y>0.
\]
Under Assumption~\ref{ass:growth_conditions}, \(B_t,B_e\to0\), and hence
\[
\liminf_{M,N\to\infty}
\widehat{\mathcal R}_{\mathrm{usr}}^{(M,N)}(\tilde Q)
\ge
Y
\]
with high probability.
\end{theorem}

Since Theorems~\ref{thm:blended_floor_usr} and~\ref{thm:misalignment_flr_meta}
are population lower bounds, they imply
\[
\liminf_{M,N\to\infty}\inf_{\tilde Q\in\mathcal A(B)}
R_{\mathrm{usr}}(\tilde Q)
\ge
Y:=\max\{Y_{\mathrm{expr}},Y_{\mathrm{obj}}\}>0.
\]

Moreover, combining this population floor with the finite-sample PAC--Bayes
bound gives, with probability at least \(1-\delta\),
\[
\widehat{\mathcal R}_{\mathrm{usr}}^{(M,N)}(\tilde Q)
\ge
Y
-
G_{\mathrm{Task}}^{\mathrm{usr}}(B_t)
-
G_{\mathrm{Env}}^{\mathrm{usr}}(B_e).
\]
Hence, under the growth conditions ensuring \(B_t,B_e\to 0\),
\[
\liminf_{M,N\to\infty}
\widehat{\mathcal R}_{\mathrm{usr}}^{(M,N)}(\tilde Q)
\ge
Y
\]
with high probability.

\begin{proof}
By Theorems~\ref{thm:blended_floor_usr} and \ref{thm:misalignment_flr_meta},
\(R_{\mathrm{usr}}(\tilde Q)\ge Y\). By Theorem~\ref{thm:user_meta_pb},
\[
R_{\mathrm{usr}}(\tilde Q)
\le
\widehat{\mathcal R}_{\mathrm{usr}}^{(M,N)}(\tilde Q)
+
G_{\mathrm{Task}}^{\mathrm{usr}}(B_t)
+
G_{\mathrm{Env}}^{\mathrm{usr}}(B_e).
\]
Rearranging gives the finite-sample lower bound. The asymptotic statements follow
from the population floor and the growth condition \(B_t,B_e\to0\).
\end{proof}

The lower bound in Theorem~\ref{cor:user_negative} is a \emph{population} statement:
the quantities \(Y_{\mathrm{expr}}\) and \(Y_{\mathrm{obj}}\) depend only on the task family, the admissibility constraints, and the information content of the prompt channel, and are independent of the sample sizes \(M\) and \(N\).
The asymptotic limit \(M,N\to\infty\) is invoked only to guarantee that the PAC--Bayes estimation terms vanish, so that the empirical risk cannot converge below the population-level floor. Importantly, this limit concerns evaluation or estimation sample sizes (i.e., $M_{eval},N_{eval}$) rather than additional training data; the population-level floors derived earlier persist even when estimation noise vanishes.
Theorem~\ref{cor:user_negative} does not assert failure for all tasks, but rather shows that
for any prompt-conditioned system with finite admissibility and information capacity,
there exist task families for which no admissible prompting strategy can eliminate the population error,
even with unlimited data.


\section{Related Works} \label{sec:related_work}

\paragraph{Learnability of Large Language Models.}
Closest to our work is \cite{wies2023learnability} who perform a PAC-based analysis of in-context learning of LLMs. Their analysis leads to a finite sample complexity analysis of the in-context learning setup of LLMs. Their results apply to an in-context learning setup in which prompts are constructed by
concatenating pairs of the task's inputs along with their labels. In this regard, our setup is more general since we consider prompts that are sampled from a prompt protocol or policy that lies within a generic User space (hypothesis class). Similar to our setup, they consider a set of latent tasks sampled from some pretraining distribution however, they consider a setup in which the pretraining data is fixed. In contrast, we consider a setting in which the learner can introduce an inductive bias into the training data set. Additionally, our results provide a generalisation bound allowing us to specify the performance error from a given finite number of samples. This provides important insights both in the finite regime case and the asymptotic case where it is revealed that the performance bound maintains an error due to possible misalignment. Indeed, an important component of our work is the consideration of the case of misspecification --- that is when the User's language hypothesis space does not allow the System to successfully infer the underlying task.   This captures practical constraints faced by the User (prompt-engineer). Lastly, our work deals with a Bayesian formulation, namely PAC-Bayes theory \cite{alquier2021User}. The corresponding results generalise the union-bound argument which enables handling a much broader parameter set topology including 
finite, infinite and continuous parameter sets.

\paragraph{PAC-Bayes Meta-Learning.}
The standard setup of a PAC-Bayes meta-learning framework comprises of two optimisations \cite{rezazadeh2022unified}. One of these is performed by a `base learner' that observes the task data and seeks to infer the model parameter that minimises the per-task expected loss (also called per-task generalisation loss). This is done for a given task hyper-parameter which is pre-selected by a so-called `Meta-learner'. The Meta-learner performs its own optimisation over the hyper-parameter which lies in some given hypothesis class. The goal of the Meta-learner is to infer the hyper-parameter that serves as a prior for learning new and as of yet, unobserved tasks that are sampled at random from some environment distribution. In particular, the goodness of the Meta-learners choice of hyper-parameter is measured by a meta-generalisation loss which the Meta-learner seeks to minimise. The Meta-learner performs the optimisation while observing the data samples generated by a set of task sampled from some random distribution.

\color{black}

\paragraph{Cheap-Talk Models.} In the strategic transmissions model (commonly referred to as the Cheap Talk model), two rational decision-makers engage in a strategic interaction \cite{farrell1996cheap,crawford1982strategic}. One of the agents, called the sender is endowed with private information about their objectives i.e. a parameter over their utility function. The Sender is allowed to communicate with the other agent, the Receiver by way of direct, one-way costless communication. Under this setup, a misalignment  between the objectives of the agents compels the Sender to introduce noise in their communication to the receiver so as to obscure informative information about their underlying objectives. In \cite{crawford1982strategic}, it is shown that the magnitude of the noise increases with the level of misalignment between the agents' preferences and, in the extreme case of complete misalignment, i.e. when the agents' interests are diametrically opposed (so they play a zero-sum game), the Sender does not communicate any useful information and communicates only noise or `babbles' to the Receiver.

The formalism we develop is based on an interaction which involves two independent `optimisers' one of which, the User, communicates to another, the System, by transmitting a context communication after observing the task. This gives rise to a game in which one of the players uses a cheap-talk protocol to send messages to the other. We formalise this in Sec. \ref{sec:game-th}.  In our model, there exists a different notion of misalignment which arises from the fact that the User, who communicates with the System does not necessarily have access to a context communication which reveals the underlying task to the System.

Theorem \ref{thm:blended_floor_usr} demonstrates that the magnitude of the misalignment error decreases with greater alignment between the User and the System. This relationship has a striking similarity to the misalignment relationship in Cheap-Talk games.

\paragraph{Relation to Non-Cooperative Games.}\label{sec:game-th}
In the previous sections, we described the formalism of the general purpose language model. At the heart of the construct are two separate optimisation process each of which is equipped with its own objective ---- this immediately suggests that the formalism can be mapped to \textit{non-cooperative game}. In this section, we now discuss the relationship between the above construction and non-cooperative games. The benefit of providing this insight comes from the fact that in non-cooperative game theory, the stable point solution of scenarios can be fully characterised using a concept known as an \textit{equilibrium} \cite{mas1995microeconomic}. To make this connection, we first note that the above construction defines an interaction between two agents. Each agent has its own distinct objective which it seeks to maximise by an appropriate choice of its decision variable. Given these remarks, the above setup can be viewed as a non-cooperative Bayesian game in which the User performs strategic communication or `cheap talk' \cite{crawford1982strategic}. 

An equilibrium or stable point is achieved when all agents use a strategy which delivers a payoff that cannot be improved on by unilaterally deviating from their current strategy (while the other agent's strategy remains fixed). Therefore, , we create a \emph{cheap talk} game \cite{farrell1996cheap} between the User and the System.

\subsection*{Future Work} \label{sec:theory_mas}
\paragraph{ Designing or Learning Optimal System Objectives.}
A natural extension of our framework is to move beyond analysing misalignment under a
fixed System objective and instead treat the System’s objective itself as a design variable.
In our analysis, the gap between $L_{\mathrm{usr}}$ and $L_{\mathrm{sys}}$, together with the limited
expressiveness of the User’s context space, determines an irreducible component of the
generalisation error. Rather than viewing $L_{\mathrm{sys}}$ as exogenously specified, one may
seek to \emph{design} or \emph{learn} an optimal System objective that minimises this
alignment gap while respecting safety constraints. Formally, this suggests introducing a parameterised family of System objectives
$\{L_{\mathrm{sys}}^{G }\}_{G  \in \mathcal{G}}$, where $G$ governs properties such as: (i) the relative
weighting placed on safety or policy constraints versus user-utility terms, (ii) the degree of
regularisation imposed through System priors or posteriors, and (iii) how the Prompt
Interpreter interprets and rewrites user prompts (e.g.\ by replacing $\pi_{R}$ with a
parameterised $\pi_{R}^{\phi}$). Under this parameterisation, one can define a
meta-objective
\begin{align}\nonumber
\mathcal{J}(\phi)
=
\mathbb{E}_{\theta \sim P_{\Theta}}
\Big[
\text{User utility at the cheap-talk equilibrium induced by }
L_{\mathrm{sys}}^{\phi}, \pi_{U}, \pi_{R}^{\phi}
\Big]&
\\- \lambda \cdot \text{SafetyPenalty}(\phi)&,
\end{align}
which seeks System objectives whose induced equilibria yield high user value while
satisfying any required safety constraints. In the PAC–Bayesian formulation developed in
this paper, the same idea may be expressed as
\[
\mathcal{J}(\phi)
=
\underbrace{
\mathbb{E}_{\theta,\,c}
\big[
L_{\mathrm{usr}}(\theta;c;\phi)
\big]
}_{\text{user-centric performance}}
\;+\;
\underbrace{
\text{terms depending on }
D_{\rm KL}(Q_{\theta}^{\phi}\Vert P_{0}^{\phi}),
\;
D_{\rm KL}(\tilde Q^{\phi}\Vert P_{U})
}_{\text{generalisation and complexity control}},
\]
thereby allowing the designer to jointly optimise the trade-off between user intent,
System-induced misalignment, and the generalisation guarantees established in this work.
Optimising $\phi$ would then yield a \emph{System objective} chosen specifically to
maximise alignment and generalisation performance under communication constraints.

This perspective naturally aligns with the mechanism-design interpretation of our model~\cite{nisan1999algorithmic, mguni2019efficient}.
Here, the System objective and the Prompt Interpreter jointly constitute a ``mechanism'' that
maps user prompts into LLM behaviour. Mechanism design offers a principled methodology
for selecting this mechanism so that the cheap-talk equilibrium has desirable properties,
such as maximal information transmission, safety-aware incentive compatibility, or
worst-case robustness. Exploring the design or learning of such an optimal System
objective---possibly using tools from differentiable mechanism design and inverse
reward modelling---is an important next step toward constructing LLM-based general
purpose solvers that are not only expressive and adaptive, but also reliably aligned with
user intent under communication and safety constraints.


\paragraph{RLHF.} A further extension is to incorporate reinforcement learning from human 
feedback (RLHF) into the cheap–talk framework. In practical LLMs, RLHF 
modifies behaviour through a reward model that reshapes the System’s 
effective objective $L^{\phi}_{\mathrm{sys}}$, penalising outputs that are 
unsafe or misaligned with human preferences. Within our formulation, this 
reward model can be treated as an additional component of the System’s 
loss, thereby altering the equilibrium between User and System: adjusting 
$L^{\phi}_{\mathrm{sys}}$ influences how the Prompt Interpreter rewrites prompts 
and how much task information is ultimately transmitted through the 
communication channel $\pi_{R} \circ \pi_{U}$. Integrating RLHF in this way 
opens the possibility of analysing how reward–model design affects both 
objective misalignment and expressivity misalignment, and how these 
interact with the PAC–Bayesian generalisation guarantees established in 
this work.

\paragraph{Multi-agent extension.}The analysis we performed in our paper considers a setting in which a single User communicates with a single System. This setup can be extended to multi-agent settings.  In this setting, the presence of different decision-making entities  leads to different optimisation objectives the implications of which were discussed in detail. In a multi-agent setting, a collective of agents each use their own language model to form messages to communicate with other agents. Introducing other agents into the System each in possession of their own knowledge and local information now leads to informational asymmetry between different agents. Additionally, in a multi-agent setting, there is the potential for the agents to have different intentions and objectives. In this setting, we now consider a set of agents $\mathcal{N}:=\{1,\ldots, M\}$ that each are equipped with their own language model. Therefore, , in this setting we need to consider a set of objectives for each agent. analogous to the single agent case, now each agent $k\in\mathcal{N}$ aims to identify an optimal context communication $c_k^*\in\mathcal{C}^k$ from a given task $\theta^k_i$ so that its generalisation loss is minimised.

\textbf{Future system design}. Our results motivate an important research direction: to expand and enrich the task interface of future models such that more task families become identifiable and information bottlenecks under realistic use cases can be minimized. Potential mechanism include: structured prompting or iterative protocols that improve expressivity and multimodal or world-grounded models that expand the space of identifiable tasks

\section{Conclusion}

Large language models have been deployed to tackle a number of general environments ranging from practical tasks such as simulated cooking \cite{yan2023ask,zhang2023proagent}, mathematical problem-solving \cite{frieder2023mathematical,cobbe2021training,poola2023guiding}, recreational games \cite{feng2023chessgpt,ma2023large}, knowledge work, legal documentation and computer programming among others \cite{kaddour2023challenges,biswas2023role}. Nevertheless, research into their sample and computational constraints has yet to fully reveal the limits of language models as problem solvers. Recent discussions on the constraints on the problem-solving capacity of (autoregressive models) classes of language models to which the most prominent language models e.g. ChatGPT, GPT3.5 belong have helped shed light on some of the constraints facing language models in their application to decision problems. Nevertheless, we argue that characterising aspects such as sample complexity within problem classes that are known to be solvable by language models advances is a critical component for understanding the limits of the precision and efficiency of algorithms within language models and paving the way for subsequent developments. A key output of this paper is to generate a deeper understanding of such characteristics of language models in settings in which they can be deployed as universal solvers. 

In this paper we studied the concept of using language models as general-purpose solvers for machine learning. To develop an understanding of the feasibility of creating a general-purpose machine learning System using
large language models and the limits of such a tool, we gave a rigorous analytic treatment of a structure that resembles a cheap-talk game between a User and the System. The System tries to infer the task from the User's prompt communication. In this analysis, we addressed the pertinent question of the sample efficiency of language models as general-purpose learners using the framework of PAC-Bayes analysis, allowing us to derive sample efficiency results for a general class of language models and their input spaces. The PAC-Bayes analyses offers
valuable insights into User-System alignment and context prompt optimisation. Specifically, our results provide bounds that quantify the generalisation loss across multiple tasks. To this end, the PAC-Bayes learning bounds derived in this paper serve to aid in initiating the study these important characteristics of language models from which further PAC-Bayes bounds can be derived. In this way, our results help establish new territories for further exploration of the capabilities of language models.  
An important component of our analysis is the study of language models as general-purpose learners when an effective communication protocol between the User and the System is absent or when there is objective misalignment.  Our analysis reveals that with full alignment, zero-shot learning becomes achievable. However, when there is a
greater misalignment or informational bottleneck, performance gap cannot be overcome even in the asymptotic data regime.

Overall, our analysis clarifies fundamental limitations to prompt-based LLM systems as reliable 
general-purpose solvers: namely, when the User–System communication channel is 
insufficiently expressive and the System’s objective is unsuitably aligned. Beyond showing the limitations of current scaling approaches, our results are diagnostic; we identify importance of expressivity and information limitation and therefore the key role they play in enhancing model performance for future systems.

These results provide a principled foundation for understanding and improving human–LLM interaction through the lens of 
information theory, cheap-talk games, and PAC–Bayesian meta-learning.

\bibliographystyle{alpha}
\bibliography{sample}

@article{mahowald2024dissociating,
  title={Dissociating language and thought in large language models},
  author={Mahowald, Kyle and Ivanova, Anna A and Blank, Idan A and Kanwisher, Nancy and Tenenbaum, Joshua B and Fedorenko, Evelina},
  journal={Trends in cognitive sciences},
  volume={28},
  number={6},
  pages={517--540},
  year={2024},
  publisher={Elsevier}
}

@article{andreas2022language,
  title={Language models as agent models},
  author={Andreas, Jacob},
  journal={arXiv preprint arXiv:2212.01681},
  year={2022}
}

@article{shanahan2023role,
  title={Role play with large language models},
  author={Shanahan, Murray and McDonell, Kyle and Reynolds, Laria},
  journal={Nature},
  volume={623},
  number={7987},
  pages={493--498},
  year={2023},
  publisher={Nature Publishing Group UK London}
}

@inproceedings{nisan1999algorithmic,
  title={Algorithmic mechanism design},
  author={Nisan, Noam and Ronen, Amir},
  booktitle={Proceedings of the thirty-first annual ACM symposium on Theory of computing},
  pages={129--140},
  year={1999}
}

@inproceedings{mguni2019efficient,
  title={Efficient reinforcement dynamic mechanism design},
  author={Mguni, David and Tomczak, Marcin},
  booktitle={GAIW: Games, agents and incentives workshops, at AAMAS, Montreal, Canada},
  year={2019}
}

@inproceedings{zhang2024llama,
  title={LLaMA-adapter: Efficient fine-tuning of large language models with zero-initialized attention},
  author={Zhang, Renrui and Han, Jiaming and Liu, Chris and Zhou, Aojun and Lu, Pan and Qiao, Yu and Li, Hongsheng and Gao, Peng},
  booktitle={The Twelfth International Conference on Learning Representations},
  year={2024}
}

@misc{taori2023stanford,
  title={Stanford alpaca: An instruction-following llama model},
  author={Taori, Rohan and Gulrajani, Ishaan and Zhang, Tianyi and Dubois, Yann and Li, Xuechen and Guestrin, Carlos and Liang, Percy and Hashimoto, Tatsunori B},
  year={2023},
  publisher={Stanford, CA, USA}
}

@article{hu2022lora,
  title={Lora: Low-rank adaptation of large language models.},
  author={Hu, Edward J and Shen, Yelong and Wallis, Phillip and Allen-Zhu, Zeyuan and Li, Yuanzhi and Wang, Shean and Wang, Lu and Chen, Weizhu and others},
  journal={ICLR},
  volume={1},
  number={2},
  pages={3},
  year={2022}
}

@article{ouyang2022training,
  title={Training language models to follow instructions with human feedback},
  author={Ouyang, Long and Wu, Jeffrey and Jiang, Xu and Almeida, Diogo and Wainwright, Carroll and Mishkin, Pamela and Zhang, Chong and Agarwal, Sandhini and Slama, Katarina and Ray, Alex and others},
  journal={Advances in neural information processing systems},
  volume={35},
  pages={27730--27744},
  year={2022}
}

@article{tamkin2021understanding,
  title={Understanding the capabilities, limitations, and societal impact of large language models},
  author={Tamkin, Alex and Brundage, Miles and Clark, Jack and Ganguli, Deep},
  journal={arXiv preprint arXiv:2102.02503},
  year={2021}
}

@article{biswas2023role,
  title={Role of ChatGPT in Computer Programming.: ChatGPT in Computer Programming.},
  author={Biswas, Som},
  journal={Mesopotamian Journal of Computer Science},
  volume={2023},
  pages={8--16},
  year={2023}
}

@article{zhang2023proagent,
  title={Proagent: Building proactive cooperative ai with large language models},
  author={Zhang, Ceyao and Yang, Kaijie and Hu, Siyi and Wang, Zihao and Li, Guanghe and Sun, Yihang and Zhang, Cheng and Zhang, Zhaowei and Liu, Anji and Zhu, Song-Chun and others},
  journal={arXiv preprint arXiv:2308.11339},
  year={2023}
}

@article{poola2023guiding,
  title={Guiding AI with human intuition for solving mathematical problems in Chat GPT},
  author={Poola, Indrasen and Bo{\v{z}}i{\'c}, Velibor},
  journal={Journal Homepage: http://www. ijmra. us},
  volume={11},
  number={07},
  year={2023}
}

@article{kaddour2023challenges,
  title={Challenges and applications of large language models},
  author={Kaddour, Jean and Harris, Joshua and Mozes, Maximilian and Bradley, Herbie and Raileanu, Roberta and McHardy, Robert},
  journal={arXiv preprint arXiv:2307.10169},
  year={2023}
}

@article{ma2023large,
  title={Large Language Models Play StarCraft II: Benchmarks and A Chain of Summarization Approach},
  author={Ma, Weiyu and Mi, Qirui and Yan, Xue and Wu, Yuqiao and Lin, Runji and Zhang, Haifeng and Wang, Jun},
  journal={arXiv preprint arXiv:2312.11865},
  year={2023}
}

@article{feng2023chessgpt,
  title={ChessGPT: Bridging Policy Learning and Language Modeling},
  author={Feng, Xidong and Luo, Yicheng and Wang, Ziyan and Tang, Hongrui and Yang, Mengyue and Shao, Kun and Mguni, David and Du, Yali and Wang, Jun},
  journal={arXiv preprint arXiv:2306.09200},
  year={2023}
}

@article{cobbe2021training,
  title={Training verifiers to solve math word problems},
  author={Cobbe, Karl and Kosaraju, Vineet and Bavarian, Mohammad and Chen, Mark and Jun, Heewoo and Kaiser, Lukasz and Plappert, Matthias and Tworek, Jerry and Hilton, Jacob and Nakano, Reiichiro and others},
  journal={arXiv preprint arXiv:2110.14168},
  year={2021}
}

@article{frieder2023mathematical,
  title={Mathematical capabilities of chatgpt},
  author={Frieder, Simon and Pinchetti, Luca and Griffiths, Ryan-Rhys and Salvatori, Tommaso and Lukasiewicz, Thomas and Petersen, Philipp Christian and Chevalier, Alexis and Berner, Julius},
  journal={arXiv preprint arXiv:2301.13867},
  year={2023}
}

@article{yan2023ask,
  title={Ask more, know better: Reinforce-Learned Prompt Questions for Decision Making with Large Language Models},
  author={Yan, Xue and Song, Yan and Cui, Xinyu and Christianos, Filippos and Zhang, Haifeng and Mguni, David Henry and Wang, Jun},
  journal={arXiv preprint arXiv:2310.18127},
  year={2023}
}

@article{thirunavukarasu2023large,
  title={Large language models in medicine},
  author={Thirunavukarasu, Arun James and Ting, Darren Shu Jeng and Elangovan, Kabilan and Gutierrez, Laura and Tan, Ting Fang and Ting, Daniel Shu Wei},
  journal={Nature medicine},
  volume={29},
  number={8},
  pages={1930--1940},
  year={2023},
  publisher={Nature Publishing Group US New York}
}

@inproceedings{yuan2022wordcraft,
  title={Wordcraft: story writing with large language models},
  author={Yuan, Ann and Coenen, Andy and Reif, Emily and Ippolito, Daphne},
  booktitle={27th International Conference on Intelligent User Interfaces},
  pages={841--852},
  year={2022}
}

@article{floridi2020gpt,
  title={GPT-3: Its nature, scope, limits, and consequences},
  author={Floridi, Luciano and Chiriatti, Massimo},
  journal={Minds and Machines},
  volume={30},
  pages={681--694},
  year={2020},
  publisher={Springer}
}

@article{lin2020limitations,
  title={Limitations of autoregressive models and their alternatives},
  author={Lin, Chu-Cheng and Jaech, Aaron and Li, Xin and Gormley, Matthew R and Eisner, Jason},
  journal={arXiv preprint arXiv:2010.11939},
  year={2020}
}

@article{alquier2021user,
  title={User-friendly introduction to PAC-Bayes bounds},
  author={Alquier, Pierre},
  journal={arXiv preprint arXiv:2110.11216},
  year={2021}
}

@article{wies2023learnability,
  title={The learnability of in-context learning},
  author={Wies, Noam and Levine, Yoav and Shashua, Amnon},
  journal={arXiv preprint arXiv:2303.07895},
  year={2023}
}

@article{masegosa2020learning,
  title={Learning under model misspecification: Applications to variational and ensemble methods},
  author={Masegosa, Andres},
  journal={Advances in Neural Information Processing Systems},
  volume={33},
  pages={5479--5491},
  year={2020}
}

@article{colson2007overview,
  title={An overview of bilevel optimization},
  author={Colson, Beno{\^\i}t and Marcotte, Patrice and Savard, Gilles},
  journal={Annals of operations research},
  volume={153},
  pages={235--256},
  year={2007},
  publisher={Springer}
}

@article{greene2008unsupervised,
  title={Unsupervised learning and clustering},
  author={Greene, Derek and Cunningham, P{\'a}draig and Mayer, Rudolf},
  journal={Machine learning techniques for multimedia: Case studies on organization and retrieval},
  pages={51--90},
  year={2008},
  publisher={Springer}
}

@article{crawford1982strategic,
  title={Strategic information transmission},
  author={Crawford, Vincent P and Sobel, Joel},
  journal={Econometrica: Journal of the Econometric Society},
  pages={1431--1451},
  year={1982},
  publisher={JSTOR}
}

@article{farrell1996cheap,
  title={Cheap talk},
  author={Farrell, Joseph and Rabin, Matthew},
  journal={Journal of Economic perspectives},
  volume={10},
  number={3},
  pages={103--118},
  year={1996},
  publisher={American Economic Association}
}

@article{xie2021explanation,
  title={An explanation of in-context learning as implicit bayesian inference},
  author={Xie, Sang Michael and Raghunathan, Aditi and Liang, Percy and Ma, Tengyu},
  journal={arXiv preprint arXiv:2111.02080},
  year={2021}
}

@book{mas1995microeconomic,
  title={Microeconomic theory},
  author={Mas-Colell, Andreu and Whinston, Michael Dennis and Green, Jerry R and others},
  volume={1},
  year={1995},
  publisher={Oxford university press New York}
}

@inproceedings{rezazadeh2022unified,
  title={A Unified View on PAC-Bayes Bounds for Meta-Learning},
  author={Rezazadeh, Arezou},
  booktitle={International Conference on Machine Learning},
  pages={18576--18595},
  year={2022},
  organization={PMLR}
}
\newpage
\appendix
\section{Standing Assumptions}

Our results are built under the following assumptions:

\begin{assumption}[System exponential-moment condition]
\label{ass:exp_moment}
There exists $\eta>0$ such that for all $G\in\mathcal G$ and all admissible $(c,\theta)$,
the random variable $\ell_{\mathrm{sys}}(G , \varphi;c,S_\theta^{(N)},X,Y)$ satisfies
\[
\mathbb E\Big[\exp\big(\eta(\ell_{\mathrm{sys}}-\mathbb E[\ell_{\mathrm{sys}}])\big)\Big]<\infty,
\]
where the expectation is over $(X,Y)\sim P_{\mathrm{task}}(\cdot\mid\theta)$ and the internal randomness of $\pi_R^\varphi$.
\end{assumption}

\begin{assumption}[User exponential-moment condition]
\label{ass:user_exp_moment}
There exists $\eta>0$ such that for all admissible $(\theta,c)$, the random variable
$\widehat{\mathcal L}_{\mathrm{usr}}^{(N)}(\theta;c\mid T_\theta^{(N)},S_\theta)$ satisfies an exponential-moment
condition, that is
the random variable $\ell_{\mathrm{usr}}(x;\theta,c)$ satisfies
\[
\mathbb E\Big[\exp\big(\eta(\ell_{\mathrm{usr}}-\mathbb E[\ell_{\mathrm{usr}}])\big)\Big]<\infty,
\]
where the expectation is over $(X,Y)\sim P_{\mathrm{task}}(\cdot\mid\theta)$.
\end{assumption}

\begin{assumption}[Finite task packing with positive mass] \label{ass:A1finite_task_packing}
There exists a finite subset \(\Theta_0=\{\theta_1,\dots,\theta_K\}\subseteq\Theta\),
with \(K\ge2\) and
\[
\alpha := P_{\mathrm{meta}}(\Theta_0) > 0,
\]
such that conditional on \(\Theta\in\Theta_0\),
the index \(U\in\{1,\dots,K\}\) defined by \(\Theta=\theta_U\) is uniform.
\end{assumption}
This is a standard identifiability condition.
If no finite subset of tasks is distinguishable in prediction,
then there is nothing to communicate and no floor should exist.

\section{Preliminary results}
We begin by proving a preliminary result which (by way of a straightforward extension) enlarges the Donsker–Varadhan variational formula:
\begin{lemma}\label{donsker-v}
For any probability
measures \(Q\ll P\) and any measurable function \(f\),
\[
\mathbb E_{h\sim Q}[f(h)]
\le
D_{\mathrm{KL}}(Q\Vert P)
+
\log \mathbb E_{h\sim P}\!\left[\exp(f(h))\right].
\]
Equivalently, for any \(\lambda>0\),
\[
\mathbb E_{h\sim Q}[f(h)]
\le
\frac{1}{\lambda}
\left\{
D_{\mathrm{KL}}(Q\Vert P)
+
\log \mathbb E_{h\sim P}\!\left[\exp(\lambda f(h))\right]
\right\}.
\]    
\end{lemma}

\begin{lemma}[Pointwise lower bound under admissibility]
\label{lem:pointwise_obj_lb}
Assume the effective System predictor is admissible:
\begin{equation}
P_{G,\varphi}(\cdot\mid x,c)\in \mathcal P_{\mathrm{safe}}(x,c)
\qquad \text{for all admissible }(G,\varphi)\in\mathcal H,\ (x,c).
\label{eq:admissible_predictor}
\end{equation}
Then for all \((\theta,x,c)\),
\begin{equation}
D_{\rm KL}\!\big(P_\theta(\cdot\mid x)\,\|\,P_{G,\varphi}(\cdot\mid x,c)\big)
\ \ge\
\Delta_{\mathrm{obj}}(\theta,x,c).
\label{eq:pointwise_lb}
\end{equation}
where for each \((\theta,x,c)\), the pointwise distance of the User-ideal distribution to the admissible set $\Delta_{\mathrm{obj}}$ is defined by:
\begin{equation}
\Delta_{\mathrm{obj}}(\theta,x,c)
:=
\inf_{Q\in\mathcal P_{\mathrm{safe}}(x,c)}
D_{\rm KL}\!\big(P_\theta(\cdot\mid x)\,\|\,Q\big).
\label{eq:Delta_obj_def}
\end{equation}
\end{lemma}

\begin{proof}
This is immediate from the definition \eqref{eq:Delta_obj_def}. Since
\(P_{G,\varphi}(\cdot\mid x,c)\in\mathcal P_{\mathrm{safe}}(x,c)\),
the infimum over \(Q\in\mathcal P_{\mathrm{safe}}(x,c)\) is at most the value attained at
\(Q=P_{G,\varphi}(\cdot\mid x,c)\), yielding \eqref{eq:pointwise_lb}.
\end{proof}


\begin{lemma}[Data Processing Inequality]\label{lemma:DPI}
Let $X,Y,Z$ be random variables on measurable spaces such that $X \to Y \to Z$ forms a Markov chain, i.e.
\[
P_{Z\mid X,Y}(\cdot\mid x,y) \;=\; P_{Z\mid Y}(\cdot\mid y)
\quad\text{for $P_{X,Y}$-a.e.\ $(x,y)$.}
\]
Equivalently, the joint law factorises as
\[
P_{X,Y,Z} \;=\; P_X\,P_{Y\mid X}\,P_{Z\mid Y}.
\]
Then the mutual information satisfies
\[
I(X;Z) \;\le\; I(X;Y).
\]
Moreover,
\[
I(X;Y) - I(X;Z) \;=\; I(X;Y\mid Z) \;\ge\; 0.
\]
\end{lemma}
\begin{lemma}[Fano's Inequality]\label{lemma:fano}
Let $W$ be a discrete random variable taking values in a finite set $\mathcal W$ with
$|\mathcal W| = M \ge 2$, and let $\hat W = g(Y)$ be any estimator of $W$ based on an
observation $Y$ (possibly randomised). Define the probability of error
\[
P_e := \mathbb P(\hat W \neq W).
\]
Then
\[
H(W \mid Y)
\;\le\;
h(P_e) + P_e \log(M-1),
\]
where $h(p) := -p\log p - (1-p)\log(1-p)$ is the binary entropy function
(with the convention $0\log 0 := 0$).
Consequently,
\[
P_e
\;\ge\;
1 - \frac{I(W;Y) + \log 2}{\log M},
\]
and, equivalently,
\[
I(W;Y)
\;\ge\;
\log M - h(P_e) - P_e \log(M-1).
\]
\end{lemma}

\begin{lemma}[Lemma F.1. \cite{rezazadeh2022unified}]\label{RezF1}
Let \(X_i\), \(i=1,\ldots, N\) be independent random variables and suppose that for a given $a_i\in \mathbb{R}_{>0}$ we have that
\begin{align}
\mathbb{P}_{X_i}\left[f_i(X_i)\geq a_i\right]\leq \delta_i,    
\end{align}
where \(\delta_i\in [0,1]\), then the following holds:
\begin{align}
    \mathbb{P}_{X_{1:N}}\left[\sum_{i\geq 1}f(X_i)\leq \sum_{i\geq 1} a_i\right]\geq 1-\sum_{i\geq 1}\delta_i
\end{align}
\end{lemma}

\begin{lemma}\label{lemma:rez-exp}
Let \(X_1,\ldots, X_m\) be independent random variable and let \(f:\cX\to\mathbb{R}\) be a sub-Gaussian function with parameter \(\sigma\). Assume \(\Delta:= \mathbb{E}\left[f(X)\right]-\frac{1}{m}\sum_{k=1}^mg(X_i)\), where for \(\epsilon > 0\), we have \(\mathbb{P}\left[\Delta \geq \epsilon\right]\leq \exp\left(\frac{-m\epsilon^2}{2\sigma^2}\right)\) then
\begin{align}
\mathbb{E}\left[e^{\lambda m\Delta^2}\right]\leq \frac{1}{1-2\lambda \sigma^2},    
\end{align}
for \(\lambda \leq \frac{1}{2\sigma^2}\).
\end{lemma}

\section{Additional Results}

\begin{lemma}[Task confusion implies predictive distortion]
\label{lem:confusion_distortion}
Let $\Theta_0=\{\theta_1,\ldots,\theta_K\}$ be a set of tasks satisfying
Assumption~\ref{ass:A2task_separation}. For any effective representation $Z$ and any
predictive kernel $Q_Z(\cdot\mid x)$ used by the Base Solver, define the decoder
\[
\widehat\Theta(Z)
\in
\arg\min_{\theta_j\in\Theta_0}
\mathbb E_x
D_{\rm KL}\!\left(
P_{\theta_j}(\cdot\mid x)\,\|\,Q_Z(\cdot\mid x)
\right).
\]
Then
\[
\mathbb E
\left[
D_{\rm KL}\!\left(
P_{\Theta}(\cdot\mid x)\,\|\,Q_Z(\cdot\mid x)
\right)
\,\middle|\,
\Theta\in\Theta_0
\right]
\ge
\delta\,
\mathbb P\!\left(
\widehat\Theta(Z)\neq \Theta
\,\middle|\,
\Theta\in\Theta_0
\right).
\]
\end{lemma}

\begin{lemma}[User misspecification decomposition]
\label{lem:user_misspec}
Define the User Bayes-optimal loss for task $\theta$ as
\[
\mathcal L_{\mathrm{usr}}^\star(\theta)
:=
\inf_{\pi(\cdot\mid\theta)}
\mathbb E_{x\sim P_{\mathrm{task}}(\cdot\mid\theta)}
\Big[
D_{\rm KL}\!\big(P_\theta(\cdot\mid x)\,\|\,P_{\pi}( \cdot\mid x)\big)
\Big],
\]
and the best-in-class loss achievable via contexts in $\mathcal C$ as $
\mathcal L_{\mathrm{usr},\mathcal C}^\star(\theta)
:=
\inf_{c\in\mathcal C}
\mathcal L_{\mathrm{usr}}(\theta;c)$.
Then for any $\tilde Q\in\Delta(\mathcal C)$ and any task $\theta$,
\begin{equation}
\mathbb E_{c\sim \tilde Q}\Big[\mathcal L_{\mathrm{usr}}(\theta;c)\Big]
=
\mathcal L_{\mathrm{usr}}^\star(\theta)
+
A_{\mathrm{usr}}(\theta)
+
E_{\tilde Q}(\theta),
\label{eq:user_misspec_decomp}
\end{equation}
where $A_{\mathrm{usr}}(\theta):=\mathcal L_{\mathrm{usr},\mathcal C}^\star(\theta)-\mathcal L_{\mathrm{usr}}^\star(\theta)\ge 0$ is an approximation (misspecification) term induced by restricting communication to $\mathcal C$,
and $E_{\tilde Q}(\theta):=\mathbb E_{c\sim \tilde Q}[\mathcal L_{\mathrm{usr}}(\theta;c)]-\mathcal L_{\mathrm{usr},\mathcal C}^\star(\theta)\ge 0$ is suboptimality relative to the best-in-class context.
\end{lemma}

\begin{proposition}[Gibbs form of the PAC-optimal User posterior]
\label{prop:gibbs_usr}
Let \(P_U\in\Delta(\mathcal C)\) be a reference measure and define
\[
\widehat r(c)
:=
\frac1M
\sum_{i=1}^{M}
\widehat{\mathcal L}_{\mathrm{usr}}^{(N)}
(\theta_i;c\mid T_i^{(N)},S_i).
\]
Consider the regularised empirical objective
\[
J_{\mathrm{usr}}(\tilde Q)
=
\mathbb E_{c\sim\tilde Q}[\widehat r(c)]
+
\frac1{\lambda M}
D_{\rm KL}(\tilde Q\Vert P_U),
\qquad
\tilde Q\ll P_U.
\]
If
\[
Z_\lambda
:=
\int_{\mathcal C}
\exp(-\lambda M \widehat r(c))\,P_U(dc)
<\infty,
\]
then the unique minimiser is
\[
\tilde Q^\star(dc)
=
\frac{
\exp(-\lambda M\widehat r(c))
}{Z_\lambda}
P_U(dc).
\]
Equivalently,
\[
\widehat r(c)
=
\widehat{\mathcal R}_{\mathrm{usr}}^{(M,N)}(\delta_c).
\]
\end{proposition}

Proposition~\ref{prop:gibbs_usr} provides a principled interpretation of
\(\tilde Q\) as an optimisation-induced User posterior, and links the User-side
bilevel programme to regularised empirical meta-risk minimisation.

\begin{proof}
For any \(\tilde Q\ll P_U\),
\[
D_{\rm KL}(\tilde Q\Vert \tilde Q^\star)
=
D_{\rm KL}(\tilde Q\Vert P_U)
+
\lambda M\,\mathbb E_{c\sim\tilde Q}[\widehat r(c)]
+
\log Z_\lambda.
\]
Rearranging,
\[
J_{\mathrm{usr}}(\tilde Q)
=
\frac1{\lambda M}
D_{\rm KL}(\tilde Q\Vert \tilde Q^\star)
-
\frac1{\lambda M}\log Z_\lambda.
\]
The second term is independent of \(\tilde Q\), and the KL divergence is uniquely
minimised at \(\tilde Q=\tilde Q^\star\). Hence \(\tilde Q^\star\) is the unique
minimiser.
\end{proof}

The following result shows that an information bottleneck is induced from the
KL-regularised optimisation objective. The rewriting mechanism \(\pi_R\) can only reduce task-relevant information via data processing.

\begin{lemma}[Penalisation induces an information budget]
\label{lem:A3_derived}
Suppose \(\Pi^\lambda\) minimises
\[
\inf_{\Pi:\Theta\to\Delta(\mathcal C)}
\left\{
R_{\mathrm{usr}}(\Pi)
+
\frac{1}{\lambda}
\mathbb E_{\theta\sim P_{\mathrm{meta}}}
D_{\rm KL}(\Pi(\cdot\mid\theta)\Vert P_U)
\right\}.
\]
Let \(C\mid\Theta=\theta\sim\Pi^\lambda(\cdot\mid\theta)\). Then
\[
I(\Theta;C)
\le
\mathbb E_{\theta}
D_{\rm KL}(\Pi^\lambda(\cdot\mid\theta)\Vert P_U)
=:B(\lambda)<\infty.
\]
Moreover, if \(\Theta\to C\to \hat C\), then
\[
I(\Theta;\hat C)\le I(\Theta;C)\le B(\lambda).
\]
\end{lemma}

\section{Main Proofs}
\subsection{Proof of Lemma~\ref{lem:confusion_distortion}}

\begin{proof}[Proof of Lemma~\ref{lem:confusion_distortion}]
Fix an effective representation value \(z\). For any task
\(\theta_j\in\Theta_0\), define the distortion
\[
d_j(z)
:=
\mathbb E_x
D_{\rm KL}\!\left(
P_{\theta_j}(\cdot\mid x)
\,\|\, 
Q_z(\cdot\mid x)
\right).
\]
By definition, the decoder selects
\[
\widehat\Theta(z)
\in
\arg\min_{\theta_j\in\Theta_0} d_j(z).
\]
Consider the event \(\{\widehat\Theta(Z)\neq \Theta\}\). On this event, if
\(\Theta=\theta_i\) and \(\widehat\Theta(Z)=\theta_j\) with \(j\neq i\), then
the optimality of the decoder implies
\[
d_j(Z)\le d_i(Z).
\]
Hence
\[
d_i(Z)
\ge
\frac12\big(d_i(Z)+d_j(Z)\big).
\]
By the task-separation assumption, no single effective representation \(Z\)
and predictive kernel \(Q_Z\) can simultaneously approximate both
\(\theta_i\) and \(\theta_j\) below level \(\delta\). Therefore,
\[
\frac12\big(d_i(Z)+d_j(Z)\big)\ge \delta,
\]
and consequently
\[
d_i(Z)\ge \delta
\qquad
\text{on the event } \{\widehat\Theta(Z)\neq \Theta\}.
\]
Equivalently,
\[
D_{\rm KL}\!\left(
P_{\Theta}(\cdot\mid x)
\,\|\, 
Q_Z(\cdot\mid x)
\right)
\ge
\delta\,
\mathbf 1\{\widehat\Theta(Z)\neq \Theta\},
\]
after averaging over \(x\). Taking expectations conditional on
\(\Theta\in\Theta_0\) gives
\[
\mathbb E
\left[
D_{\rm KL}\!\left(
P_{\Theta}(\cdot\mid x)
\,\|\, 
Q_Z(\cdot\mid x)
\right)
\,\middle|\,
\Theta\in\Theta_0
\right]
\ge
\delta\,
\mathbb P\!\left(
\widehat\Theta(Z)\neq \Theta
\,\middle|\,
\Theta\in\Theta_0
\right).
\]
This proves the claim.
\end{proof}

\subsection{Proof of Lemma \ref{lem:A3_derived}}
\begin{proof}[Proof of Lemma \ref{lem:A3_derived}]
Let \(\Theta\sim P_{\mathrm{meta}}\) and \(C\mid \Theta=\theta \sim \tilde Q^\lambda(\cdot\mid \theta)\).
Write the joint law as
\[
P_{\Theta,C}(\mathrm d\theta,\mathrm dc)
=
P_{\mathrm{meta}}(\mathrm d\theta)\,\tilde Q^\lambda(\mathrm dc\mid \theta),
\]
and denote the marginal of \(C\) by
\[
P_C(\mathrm dc)
:=
\int_\Theta P_{\mathrm{meta}}(\mathrm d\theta)\,\tilde Q^\lambda(\mathrm dc\mid \theta).
\]
Recall the standard identity
\begin{equation}
I(\Theta;C)
=
\mathbb E_{\theta\sim P_{\mathrm{meta}}}
\Big[
D_{\rm KL}\!\big(\tilde Q^\lambda(\cdot\mid\theta)\,\|\,P_C\big)
\Big].
\label{eq:mutual_info_as_expected_KL}
\end{equation}
Moreover, for any reference distribution \(P_U\in\Delta(\mathcal C)\), we have the decomposition
\begin{equation}
\mathbb E_{\theta}D_{\rm KL}\!\big(\tilde Q^\lambda(\cdot\mid\theta)\,\|\,P_U\big)
=
I(\Theta;C)
+
D_{\rm KL}(P_C\|P_U),
\label{eq:KL_decomp_info}
\end{equation}
which can be verified by expanding the KL divergences under a common (dominating) measure.
Since \(D_{\rm KL}(P_C\|P_U)\ge 0\), \eqref{eq:KL_decomp_info} implies
\begin{equation}
I(\Theta;C)
\ \le\
\mathbb E_{\theta}
D_{\rm KL}\!\big(\tilde Q^\lambda(\cdot\mid\theta)\,\|\,P_U\big).
\label{eq:info_budget_bound}
\end{equation}
Define
\(
B(\lambda)
:=
\mathbb E_{\theta}
D_{\rm KL}\!\big(\tilde Q^\lambda(\cdot\mid\theta)\,\|\,P_U\big).
\)
Now since the penalised objective is finite at its minimiser;
\(B(\lambda)<\infty\), yielding the claimed information budget \(I(\Theta;C)\le B(\lambda)\). Finally, by construction of the prompt rewriting mechanism, \(\Theta\to C\to \hat C\) forms a Markov chain,
by Lemma~\ref{lemma:DPI}, we find that
\[
I(\Theta;\hat C)\ \le\ I(\Theta;C).
\]
Combining with \eqref{eq:info_budget_bound} yields
\(
I(\Theta;\hat C)\le I(\Theta;C)\le B(\lambda).
\)
\end{proof}

\subsection{Proof of Theorem~\ref{thm:blended_floor_usr}}

\paragraph{Evaluation law.}
Unless otherwise stated, expectations defining
$\mathcal L_{\mathrm{usr}}$ and
$\widehat{\mathcal L}_{\mathrm{usr}}^{(N)}$
are taken with respect to samples

\[
x_i\sim\mu_\theta,
\qquad
y_i\sim P_\theta(\cdot\mid x_i),
\]

so that the empirical User loss is an unbiased estimator of the
User-level KL risk.

\begin{proof}[Proof of Theorem~\ref{thm:blended_floor_usr}]

Our first step is to lower bound \(R_{\mathrm{usr}}(\Pi)\) by expected distortion.
Now, we first note that 
for all \((\theta,c)\),
\[
\mathcal L_{\mathrm{usr}}(\theta;c)
\;\ge\;
\mathbb E_{\hat c\sim\pi_R(\cdot\mid c,S)}\big[d(\theta,\hat c)\big].\label{eqn:A6}
\]
This follows directly from the definition of \(d\) and since the realised system cannot outperform the best achievable predictor
given the same information. Taking the expectation over \(c\sim\tilde Q(\cdot\mid\theta)\) gives, for each \(\theta\),
\[
\mathcal R_{\mathrm{usr}}(\theta;\tilde Q)
=
\mathbb E_{c\sim\tilde Q(\cdot\mid\theta)}[\mathcal L_{\mathrm{usr}}(\theta;c)]
\ \ge\
\mathbb E_{c\sim\tilde Q(\cdot\mid\theta)}\mathbb E_{\hat c\sim\pi_R(\cdot\mid c)}[d(\theta,\hat c)].
\]
We next take the expectation over \(\theta\sim P_{\mathrm{meta}}\) and use the generative structure
\(\Theta\to C\to \hat C\) to obtain
\begin{equation}
R_{\mathrm{usr}}(\tilde Q)
\ \ge\
\mathbb E\big[d(\Theta,\hat C)\big].
\label{eq:R_ge_Ed}
\end{equation}

Next, restrict to a finite packing \(\Theta_0\) and define the index \(U\). To this end,
let \(\Theta_0=\{\theta_1,\dots,\theta_K\}\) be as in Assumption~\ref{ass:A1finite_task_packing}, with \(\alpha=P_{\mathrm{meta}}(\Theta_0)\).
Condition on the event \(\{\Theta\in\Theta_0\}\) and define \(U\in\{1,\dots,K\}\) by \(\Theta=\theta_U\).
Then, by the law of total expectation,
\begin{equation}
\mathbb E[d(\Theta,\hat C)]
\ \ge\
\mathbb P(\Theta\in\Theta_0)\ \mathbb E\big[d(\Theta,\hat C)\,\big|\,\Theta\in\Theta_0\big]
\ =\
\alpha\ \mathbb E\big[d(\theta_U,\hat C)\big].
\label{eq:restrict_to_packing}
\end{equation}

Now, fix \(\hat c\in\hat{\mathcal C}\), and denote the posterior weights
\[
p_i(\hat c)\ :=\ \mathbb P(U=i\mid \hat C=\hat c),\qquad i=1,\dots,K.
\]
Let \(i^\star\in\arg\max_i p_i(\hat c)\) and note that \(1-\max_i p_i(\hat c)=\sum_{j\neq i^\star}p_j(\hat c)\). For any fixed \(\hat c\), let \((G,\varphi)\) be arbitrary. Then
\begin{align}
&\sum_{i=1}^K p_i(\hat c)\,
\mathbb E_x D_{\rm KL}\!\big(P_{\theta_i}(\cdot\mid x)\,\|\,P_{G,\varphi}(\cdot\mid x,\hat c)\big)
\\&\ \ge\
\sum_{j\neq i^\star} p_j(\hat c)\,
\frac{1}{2}\Big(
\mathbb E_x D_{\rm KL}(P_{\theta_{i^\star}}\|P_{G,\varphi}(\cdot\mid x,\hat c))
+
\mathbb E_x D_{\rm KL}(P_{\theta_j}\|P_{G,\varphi}(\cdot\mid x,\hat c))
\Big),
\end{align}
since \(p_{i^\star}(\hat c)\ge p_j(\hat c)\) for all \(j\neq i^\star\).
Taking \(\inf_{(G,\varphi)\in\mathcal H}\) on both sides and using Assumtion~\ref{ass:A2task_separation} yields
\begin{equation}
\sum_{i=1}^K p_i(\hat c)\, d(\theta_i,\hat c)
\ \ge\
\sum_{j\neq i^\star} p_j(\hat c)\, \delta
\ =\
\delta\big(1-\max_i p_i(\hat c)\big).
\label{eq:mixing_lb}
\end{equation}

After taking the expectation of \eqref{eq:mixing_lb} over \(\hat C\), the left-hand side becomes
\[
\mathbb E_{\hat C}\sum_{i=1}^K p_i(\hat C)\, d(\theta_i,\hat C)
=
\mathbb E\big[d(\theta_U,\hat C)\big],
\]
by the tower property, and the right-hand side becomes
\[
\delta\ \mathbb E_{\hat C}\big[1-\max_i p_i(\hat C)\big]
=
\delta\,p_{\mathrm{err}},
\]
where
\[
p_{\mathrm{err}}
:=
\inf_{\hat U(\hat C)}\mathbb P(\hat U\neq U)
=
\mathbb E_{\hat C}\big[1-\max_i p_i(\hat C)\big]
\]
is the Bayes error of classifying \(U\) from \(\hat C\).
Therefore
\begin{equation}
\mathbb E\big[d(\theta_U,\hat C)\big]
\ \ge\
\delta\,p_{\mathrm{err}}.
\label{eq:Ed_ge_delta_perr}
\end{equation}

\medskip
We now use the Fano inequality (Lemma \ref{lemma:fano}) to upper bound $p_{\mathrm{err}}$ under the information constraint. Firstly, note that \(I(U;\hat C)\le I(\Theta;\hat C)\). By Lemma~\ref{lem:A3_derived} we have that, \(I(\Theta;\hat C)\le B\), hence
\(
I(U;\hat C)\le B
\). We now apply
Fano's inequality for \(U\in\{1,\dots,K\}\) which gives
\begin{equation}
p_{\mathrm{err}}
\ \ge\
1-\frac{I(U;\hat C)+\log 2}{\log K}
\ \ge\
1-\frac{B+\log 2}{\log K}.
\label{eq:Fano_usr}
\end{equation}

To obtain the result we now substitute \eqref{eq:Fano_usr} into \eqref{eq:Ed_ge_delta_perr}, then into \eqref{eq:restrict_to_packing}, then into
\eqref{eq:R_ge_Ed}, to obtain
\[
R_{\mathrm{usr}}(\tilde Q)
\ \ge\
\mathbb E[d(\Theta,\hat C)]
\ \ge\
\alpha\,\mathbb E[d(\theta_U,\hat C)]
\ \ge\
\alpha\,\delta\left(1-\frac{B+\log 2}{\log K}\right).
\]
Since \(R_{\mathrm{usr}}(\tilde Q)\ge 0\) by definition of KL, we may write the bound with \((\cdot)_+\) which produces the required result.
\end{proof}
\begin{corollary}[Irreducible User-level error]
\label{cor:user_irreducible}
If the communication channel is not task-sufficient, then there exists $Y>0$ such that for any admissible User strategy $\tilde Q$,
\begin{equation}
\liminf_{N\to\infty}
\mathcal R_{\mathrm{usr}}(\tilde Q)
\;\ge\;
Y.
\label{eq:user_irreducible_floor}
\end{equation}
\end{corollary}

\subsection{Proof of Theorem~\ref{thm:misalignment_flr_meta}}
\begin{proof}[Proof of Theorem \ref{thm:misalignment_flr_meta}]
By Lemma~\ref{lem:pointwise_obj_lb}, for any \((G,\varphi)\in\mathcal H\),
\[
D_{\rm KL}\!\big(P_\theta(\cdot\mid x)\,\|\,P_{G,\varphi}(\cdot\mid x,c)\big)
\ \ge\
\Delta_{\mathrm{obj}}(\theta,x,c).
\]
Take expectation over \(x\sim P_{\mathrm{task}}(\cdot\mid\theta)\), then expectation over \(c\sim\tilde Q\), and finally
\(\theta\sim P_{\mathrm{meta}}\), to obtain \eqref{eq:Rusr_ge_objgap}. For the second statement, \(\Delta_{\mathrm{obj}}(\Theta,X,C)\ge 0\) almost surely, hence
\[
\mathbb E[\Delta_{\mathrm{obj}}(\Theta,X,C)]
\ \ge\
\varepsilon\,
\mathbb P\big(\Delta_{\mathrm{obj}}(\Theta,X,C)\ge \varepsilon\big)
\ \ge\ \varepsilon\,\beta,
\]
which combined with \eqref{eq:Rusr_ge_objgap} yields \eqref{eq:Yobj_positive}.
\end{proof}

\subsection{Proof of Theorem \ref{thm:reduction_rezazadeh}}
\begin{proof}[Proof of Theorem \ref{thm:reduction_rezazadeh}]
Since \(I(\Theta;Z)=H(\Theta)\), we have
\[
H(\Theta\mid Z)=H(\Theta)-I(\Theta;Z)=0.
\]
Hence \(\Theta\) is identifiable from \(Z\) almost surely. Therefore, there exists a measurable decoder
\[
g:\mathcal Z\to\Theta
\]
such that
\[
g(Z)=\Theta
\qquad\text{almost surely}.
\]
Therefore,  the observable context contains all task-relevant information. In particular, conditioning on \(Z\) is equivalent to conditioning on the latent task \(\Theta\), up to null sets.

It follows that the Prompt Interpreter need not perform any lossy inference over tasks. The System can condition its Base Solver on the recovered task \(g(Z)\), so that for each \(\theta\),
\[
P_{G,\varphi}(\cdot\mid x,Z)
=
P_G(\cdot\mid x,\theta)
\qquad
\text{whenever } g(Z)=\theta .
\]
Therefore,  there is no representational or communication-induced obstruction to task recovery. Consequently, the expressivity floor vanishes:
\[
Y_{\mathrm{expr}}=0.
\]

Next, by objective alignment,
\[
P^{\mathrm{usr}}_{\theta}(\cdot\mid x)
=
P^{\mathrm{sys}}_{\theta}(\cdot\mid x)
\qquad
\text{for all }(\theta,x).
\]
Therefore the User-ideal predictive distribution lies inside the System-admissible predictive class. Hence the pointwise objective distortion satisfies
\[
\Delta_{\mathrm{obj}}(\theta,x,z)
:=
\inf_{Q\in\mathcal P_{\mathrm{sys}}(x,z)}
D_{\mathrm{KL}}
\!\left(
P^{\mathrm{usr}}_{\theta}(\cdot\mid x)
\,\middle\|\,
Q
\right)
=0,
\]
because \(Q=P^{\mathrm{sys}}_{\theta}(\cdot\mid x)=P^{\mathrm{usr}}_{\theta}(\cdot\mid x)\) is admissible. Taking expectations gives
\[
Y_{\mathrm{obj}}
=
\mathbb E_{\theta,z,x}
\big[
\Delta_{\mathrm{obj}}(\theta,x,z)
\big]
=0.
\]

It remains only to identify the resulting learning problem. Since \(Z\) fully reveals
\(\Theta\), the User no longer faces a communication-induced task-aliasing problem.
Since the objectives are aligned, the User-ideal predictive distribution lies in the
System-admissible class. Hence
\[
Y_{\mathrm{expr}}=Y_{\mathrm{obj}}=0.
\]
Consequently, Theorems~\ref{lem:task_pb}--\ref{cor:user_negative} reduce to
estimation-only PAC--Bayes bounds, with no irreducible communication or
objective-misalignment floor.

The resulting PAC--Bayes bound contains only the standard environment-level and task-level estimation terms:
\[
R_{\mathrm{usr}}(\tilde Q)
\le
\widehat R_{\mathrm{usr}}^{(M,N)}(\tilde Q)
+
F_{\mathrm{Env}}^{\mathrm{usr}}(M)
+
F_{\mathrm{Task}}^{\mathrm{usr}}(N).
\]
Since \(Y_{\mathrm{expr}}=Y_{\mathrm{obj}}=0\), there is no additional asymptotic floor. Therefore,  the GPAI bounds reduce exactly to the standard PAC--Bayes meta-learning bounds.
\end{proof}

\subsection{Proof of  Theorem \ref{lem:task_pb}}
In advance of the proof of the theorem, we state the following assumptions:

\begin{assumption}[Task-level exponential moment]
\label{ass:task_exp_moment_usr}
There exists a finite constant \(C_{\rm Task}>0\) such that, for every task
\(\theta\) and fixed support information \(S_\theta\),
\[
\mathbb E_{T_\theta^{(N)}}
\mathbb E_{c\sim P_U}
\exp\!\left(
N
F_{\mathrm{Task}}^{\mathrm{usr}}
\left(
\mathcal L_{\mathrm{usr}}(\theta;c\mid S_\theta),
\widehat{\mathcal L}_{\mathrm{usr}}^{(N)}
(\theta;c\mid T_\theta^{(N)},S_\theta)
\right)
\right)
\le C_{\rm Task}.
\]
\end{assumption}
\begin{assumption}[Affine transformation for User comparison functions]
\label{ass:usr_affine_transform}
The functions \(F_{\mathrm{Task}}^{\mathrm{usr}}\) and
\(F_{\mathrm{Env}}^{\mathrm{usr}}\) are convex. Moreover, there exist
constants \(k_t^{\mathrm{usr}},k_e^{\mathrm{usr}}>0\) and nondecreasing
functions \(G_{\mathrm{Task}}^{\mathrm{usr}}\) and \(G_{\mathrm{Env}}^{\mathrm{usr}}\)
such that
\[
F_{\mathrm{Task}}^{\mathrm{usr}}(a,b)\le c_{\mathrm{tsk}}
\quad\Longrightarrow\quad
a\le k_t^{\mathrm{usr}} b+
G_{\mathrm{Task}}^{\mathrm{usr}}(c_{\mathrm{tsk}}),
\]
and
\[
F_{\mathrm{Env}}^{\mathrm{usr}}(a,b)\le c_{\mathrm{env}}
\quad\Longrightarrow\quad
a\le k_e^{\mathrm{usr}} b+
G_{\mathrm{Env}}^{\mathrm{usr}}(c_{\mathrm{env}}).
\]
\end{assumption}

\begin{assumption}[Environment-level sub-Gaussian User loss]
\label{ass:env_subgaussian_usr}
There exists \(\sigma_e>0\) such that, for every \(c\in\mathcal C\), the random variable
\[
\mathcal L_{\mathrm{usr}}(\Theta;c),
\qquad \Theta\sim P_{\mathrm{meta}},
\]
is \(\sigma_e^2\)-sub-Gaussian around its mean. Equivalently, for all \(t\in\mathbb R\),
\[
\mathbb E_{\Theta\sim P_{\mathrm{meta}}}
\exp\!\left(
t\big(
\mathcal L_{\mathrm{usr}}(\Theta;c)
-
\mathcal R_{\mathrm{usr}}(c)
\big)
\right)
\le
\exp\!\left(
\frac{\sigma_e^2t^2}{2}
\right),
\]
where
\[
\mathcal R_{\mathrm{usr}}(c)
:=
\mathbb E_{\Theta\sim P_{\mathrm{meta}}}
\big[
\mathcal L_{\mathrm{usr}}(\Theta;c)
\big].
\]
\end{assumption}

\begin{lemma}[Environment exponential moment from sub-Gaussianity]
\label{lem:env_exp_from_subgaussian_usr}
Suppose Assumption~\ref{ass:env_subgaussian_usr} holds. Let
\[
\widehat{\mathcal R}_{\mathrm{usr}}^{(M)}(c)
:=
\frac1M\sum_{i=1}^{M}
\mathcal L_{\mathrm{usr}}(\theta_i;c),
\qquad
\theta_i\stackrel{\mathrm{i.i.d.}}{\sim}P_{\mathrm{meta}}.
\]
For
\[
F_{\mathrm{Env}}^{\mathrm{usr}}(a,b)
=
\lambda_e(a-b)^2,
\qquad
0<\lambda_e<\frac{1}{2\sigma_e^2},
\]
we have
\[
\mathbb E_{\theta_1,\ldots,\theta_M\sim P_{\mathrm{meta}}}
\mathbb E_{c\sim P_U}
\exp\!\left(
M
F_{\mathrm{Env}}^{\mathrm{usr}}
\left(
\mathcal R_{\mathrm{usr}}(c),
\widehat{\mathcal R}_{\mathrm{usr}}^{(M)}(c)
\right)
\right)
\le
C_{\mathrm{Env}},
\]
where
\[
C_{\mathrm{Env}}
=
\frac{1}{\sqrt{1-2\lambda_e\sigma_e^2}}.
\]
\end{lemma}

\begin{proof}
Fix \(c\in\mathcal C\) and define
\[
X_i(c)
:=
\mathcal L_{\mathrm{usr}}(\theta_i;c),
\qquad
\mu(c)
:=
\mathcal R_{\mathrm{usr}}(c).
\]
By Assumption~\ref{ass:env_subgaussian_usr}, \(X_i(c)-\mu(c)\) is
\(\sigma_e^2\)-sub-Gaussian. Hence
\[
\mu(c)-\widehat{\mathcal R}_{\mathrm{usr}}^{(M)}(c)
=
-\frac1M\sum_{i=1}^{M}(X_i(c)-\mu(c))
\]
is \(\sigma_e^2/M\)-sub-Gaussian. Therefore, using the standard quadratic
exponential moment bound for a centred sub-Gaussian random variable,
\[
\mathbb E_{\theta_1,\ldots,\theta_M}
\exp\!\left(
\lambda_e M
\left(
\mu(c)-\widehat{\mathcal R}_{\mathrm{usr}}^{(M)}(c)
\right)^2
\right)
\le
\frac{1}{\sqrt{1-2\lambda_e\sigma_e^2}},
\]
for \(0<\lambda_e<1/(2\sigma_e^2)\). Averaging both sides over
\(c\sim P_U\) gives the result.
\end{proof}
\begin{proof}[Proof of Theorem \ref{lem:task_pb}]

Fix a task \(\theta\).

By definition,

\[
R_{\mathrm{usr}}(\theta;\tilde Q)
=
\mathbb E_{c\sim\tilde Q}
\Big[
\mathcal L_{\mathrm{usr}}(\theta;c)
\Big]
\]

and

\[
\widehat R_{\mathrm{usr}}^{(N)}
(\theta;\tilde Q\mid S_\theta^{(N)})
=
\mathbb E_{c\sim\tilde Q}
\Big[
\mathcal L_{\rm usr}(\theta;c),
\widehat{\mathcal L}_{\rm usr}^{(N)}(\theta;c\mid T_\theta^{(N)},S_\theta)
\Big].
\]

Since \(F_{\rm Task}^{\rm usr}\) is convex,

\[
F_{\rm Task}^{\rm usr}
\Big(
R_{\rm usr}(\theta;\tilde Q),
\widehat R_{\rm usr}^{(N)}
(\theta;\tilde Q\mid S_\theta^{(N)})
\Big)
\]

\[
\le
\mathbb E_{c\sim\tilde Q}
F_{\rm Task}^{\rm usr}
\Big(
\mathcal L_{\rm usr}(\theta;c),
\widehat{\mathcal L}_{\rm usr}^{(N)}(\theta;c\mid T_\theta^{(N)},S_\theta)
\Big).
\]

Applying the Donsker--Varadhan variational inequality (Lemma~\ref{donsker-v}) with

\[
Q=\tilde Q,
\qquad
P=P_U,
\]

and

\[
f(c)
=
N
F_{\rm Task}^{\rm usr}
\Big(
\mathcal L_{\rm usr}(\theta;c),
\widehat{\mathcal L}_{\rm usr}^{(N)}(\theta;c\mid T_\theta^{(N)},S_\theta)
\Big),
\]

gives

\[
F_{\rm Task}^{\rm usr}
\Big(
R_{\rm usr}(\theta;\tilde Q),
\widehat R_{\rm usr}^{(N)}
(\theta;\tilde Q\mid S_\theta^{(N)})
\Big)
\]

\[
\le
\frac1N
\Bigg(
D_{\rm KL}(\tilde Q\Vert P_U)
+
\log Z_\theta
\Bigg),
\]

where

\[
Z_\theta
=
\mathbb E_{c\sim P_U}
\exp
\Big(
N
F_{\rm Task}^{\rm usr}
(
\mathcal L_{\rm usr}(\theta;c),
\widehat{\mathcal L}_{\rm usr}^{(N)}(\theta;c\mid T_\theta^{(N)},S_\theta)
)
\Big).
\]

Using Assumption~\ref{ass:task_exp_moment_usr} and Markov's inequality,

\[
\mathbb P
\left(
Z_\theta
\le
\frac{C_{\rm Task}}{\delta}
\right)
\ge
1-\delta.
\]

Hence with probability at least \(1-\delta\),
\begin{align}
F_{\rm Task}^{\rm usr}
\Big(
R_{\rm usr}(\theta;\tilde Q),
\widehat R_{\rm usr}^{(N)}
(\theta;\tilde Q\mid S_\theta^{(N)})
\Big) \le
\frac{
D_{\rm KL}(\tilde Q\Vert P_U)
+
\log(C_{\rm Task}/\delta)
}
{N}.\label{eq:risk_bound_non_sp}
\end{align}

Applying Assumption~\ref{ass:usr_affine_transform}
yields

\[
R_{\rm usr}(\theta;\tilde Q)
\le
k_t^{\rm usr}
\widehat R_{\rm usr}^{(N)}
(\theta;\tilde Q\mid S_\theta^{(N)})
+
G_{\rm Task}^{\rm usr}(B_t),
\]

where

\[
B_t
=
\frac{
D_{\rm KL}(\tilde Q\Vert P_U)
+
\log(C_{\rm Task}/\delta)
}
{N}.
\]

Moreover, to obtain the explicit bound we specialise the  task-level comparison function to
\[
F_{\mathrm{Task}}^{\mathrm{usr}}(a,b)=\lambda_t(a-b)^2,
\]
then \(k_t^{\mathrm{usr}}=1\) and
\(
G_{\mathrm{Task}}^{\mathrm{usr}}(u)=\sqrt{\frac{u}{\lambda_t}}.
\)
Therefore,
\begin{align}
R_{\mathrm{usr}}(\theta;\tilde Q)
\le
\widehat R_{\mathrm{usr}}^{(N)}
(\theta;\tilde Q\mid T_\theta^{(N)},S_\theta)
+
\sqrt{
\frac{
D_{\rm KL}(\tilde Q\Vert P_U)
+
\log(C_{\rm Task}/\delta)
}{
\lambda_t N
}
}.\label{thrm4_specific_bound}
\end{align}
\end{proof}

\subsection{Proof of  Theorem \ref{thm:user_meta_pb}}





\begin{proof}[Proof of Theorem \ref{thm:user_meta_pb}]
We prove the result in two steps. The first step controls the gap between the
task-level population User losses and their empirical estimators. The second step
controls the gap between the sampled-task population average and the full
meta-population User risk.

\paragraph{Step 1: task-level generalisation.}
Fix a sampled task \(\theta_i\) and its dataset \(S_i=S_{\theta_i}^{(N)}\).
For each context \(c\in\mathcal C\), define
\[
L_i(c):=
\mathcal L_{\mathrm{usr}}(\theta_i;c),
\qquad
\widehat L_i(c):=
\widehat{\mathcal L}_{\mathrm{usr}}^{(N)}
(\theta_i;c\mid S_i).
\]
By convexity of \(F_{\mathrm{Task}}^{\mathrm{usr}}\), Jensen's inequality gives
\[
F_{\mathrm{Task}}^{\mathrm{usr}}
\left(
\mathbb E_{c\sim\tilde Q}L_i(c),
\mathbb E_{c\sim\tilde Q}\widehat L_i(c)
\right)
\le
\mathbb E_{c\sim\tilde Q}
F_{\mathrm{Task}}^{\mathrm{usr}}
\left(
L_i(c),
\widehat L_i(c)
\right).
\]
Recall by Lemma~\ref{donsker-v} (Donsker--Varadhan change-of-measure inequality): for probability
measures \(Q\ll P\) and any measurable function \(f\),
\[
\mathbb E_{h\sim Q}[f(h)]
\le
D_{\mathrm{KL}}(Q\Vert P)
+
\log
\mathbb E_{h\sim P}
\left[
\exp(f(h))
\right].
\]
Apply the inequality with
\[
Q=\tilde Q,\qquad P=P_U,
\]
and
\[
f(c)
=
N F_{\mathrm{Task}}^{\mathrm{usr}}
\left(
L_i(c),
\widehat L_i(c)
\right),
\]
yields
\[
\mathbb E_{c\sim\tilde Q}
F_{\mathrm{Task}}^{\mathrm{usr}}
\left(
L_i(c),
\widehat L_i(c)
\right)
\le
\frac1N
\left[
D_{\mathrm{KL}}(\tilde Q\Vert P_U)
+
\log
\mathbb E_{c\sim P_U}
\operatorname{e}^{
\left(
N F_{\mathrm{Task}}^{\mathrm{usr}}
\left(
L_i(c),
\widehat L_i(c)
\right)
\right)}
\right].
\]
Combining the two gives
\[
F_{\mathrm{Task}}^{\mathrm{usr}}
\left(
\mathcal R_{\mathrm{usr}}(\theta_i;\tilde Q),
\widehat{\mathcal R}_{\mathrm{usr}}^{(N)}
(\theta_i;\tilde Q\mid S_i)
\right)
\le
\frac1N
\left[
D_{\mathrm{KL}}(\tilde Q\Vert P_U)
+
\log Z_i
\right],
\]
where
\[
Z_i
:=
\mathbb E_{c\sim P_U}
\operatorname{e}^{
\left(
N F_{\mathrm{Task}}^{\mathrm{usr}}
\left(
\mathcal L_{\mathrm{usr}}(\theta_i;c),
\widehat{\mathcal L}_{\mathrm{usr}}^{(N)}
(\theta_i;c\mid S_i)
\right)
\right)}.
\]

By Assumption~\ref{ass:task_exp_moment_usr} and the Markov inequality,
\[
\mathbb P
\left(
Z_i\le \frac{2M}{\delta}
\right)
\ge
1-\frac{\delta}{2M}.
\]
Taking a union bound over \(i=1,\ldots,M\), with probability at least
\(1-\delta/2\), for all sampled tasks simultaneously,
\[
\log Z_i\le \log\frac{2M}{\delta}.
\]
On this event,
\[
F_{\mathrm{Task}}^{\mathrm{usr}}
\left(
\mathcal R_{\mathrm{usr}}(\theta_i;\tilde Q),
\widehat{\mathcal R}_{\mathrm{usr}}^{(N)}
(\theta_i;\tilde Q\mid S_i)
\right)
\le
\frac{
D_{\mathrm{KL}}(\tilde Q\Vert P_U)
+
\log\frac{2M}{\delta}
}{N}.
\]
Using the calibration property of
\(F_{\mathrm{Task}}^{\mathrm{usr}}\), we obtain
\begin{align}
\mathcal R_{\mathrm{usr}}(\theta_i;\tilde Q)
\le
\widehat{\mathcal R}_{\mathrm{usr}}^{(N)}
(\theta_i;\tilde Q\mid S_i)
+
F_{\mathrm{Task}}^{\mathrm{usr}}
\left(
\frac{
D_{\mathrm{KL}}(\tilde Q\Vert P_U)
+
\log\frac{2M}{\delta}
}{N}
\right).\label{risk_inequality_generic_interm}
\end{align}
Specifically, we have that
\begin{align}
\mathcal R_{\mathrm{usr}}(\theta_i;\tilde Q)
\leq k^{\rm usr}_e
\widehat{\mathcal R}_{\mathrm{usr}}^{(N)}
(\theta_i;\tilde Q\mid S_i) +G_{\mathrm{env}}^{\rm usr}(B^{\rm usr}_e),
\end{align}
where
\begin{align*}
    B^{\rm usr}_t:= \frac{1}{\theta_{\rm usr, env}}D_{\mathrm{KL}}(\tilde Q\Vert P_U)+\operatorname{log}\left(\frac{\mathbb E_{c\sim P_U}\left[\Gamma_e^{\rm usr\frac{1}{\theta_{\rm usr, env}}}\right]}{\delta}\right), \quad \Gamma_e^{{\rm us}}:=\operatorname{e}^{\left(
N F_{\mathrm{Task}}^{\mathrm{usr}}
\left(
L_i(c),
\widehat L_i(c)
\right)
\right)}.
\end{align*}
Averaging \eqref{risk_inequality_generic_interm} over \(i=1,\ldots,M\) gives
\[
\frac1M\sum_{i=1}^{M}
\mathcal R_{\mathrm{usr}}(\theta_i;\tilde Q)
\le
\widehat{\mathcal R}_{\mathrm{usr}}^{(M,N)}(\tilde Q)
+
F_{\mathrm{Task}}^{\mathrm{usr}}
\left(
\frac{
D_{\mathrm{KL}}(\tilde Q\Vert P_U)
+
\log\frac{2M}{\delta}
}{N}
\right),
\]
where we used the definition
\[
\widehat{\mathcal R}_{\mathrm{usr}}^{(M,N)}(\tilde Q)
=
\frac1M\sum_{i=1}^{M}
\widehat{\mathcal R}_{\mathrm{usr}}^{(N)}
(\theta_i;\tilde Q\mid S_i).
\]
Let us now define the sampled-task population average as
\[
\bar{\mathcal R}_{\mathrm{usr}}^{(M)}(\tilde Q)
:=
\frac1M\sum_{i=1}^{M}
\mathcal R_{\mathrm{usr}}(\theta_i;\tilde Q).
\]
Therefore, we now obtain with at least probability \(1-\delta/2\) that
\[
\bar{\mathcal R}_{\mathrm{usr}}^{(M)}(\tilde Q)
\le
\widehat{\mathcal R}_{\mathrm{usr}}^{(M,N)}(\tilde Q)
+
F_{\mathrm{Task}}^{\mathrm{usr}}
\left(
\frac{
D_{\mathrm{KL}}(\tilde Q\Vert P_U)
+
\log\frac{2M}{\delta}
}{N}
\right).
\]

\paragraph{Step 2: environment-level generalisation.}
For each fixed context \(c\), define the environment population and empirical
task-average losses
\[
\mathcal R_{\mathrm{usr}}(c)
:=
\mathbb E_{\theta\sim P_{\mathrm{meta}}}
\mathcal L_{\mathrm{usr}}(\theta;c),
\qquad
\widehat{\mathcal R}_{\mathrm{usr}}^{(M)}(c)
:=
\frac1M\sum_{i=1}^{M}
\mathcal L_{\mathrm{usr}}(\theta_i;c).
\]
By the definitions of \(R_{\mathrm{usr}}(\tilde Q)\) and
\(\bar{\mathcal R}_{\mathrm{usr}}^{(M)}(\tilde Q)\),
\[
R_{\mathrm{usr}}(\tilde Q)
=
\mathbb E_{c\sim\tilde Q}
\mathcal R_{\mathrm{usr}}(c),
\qquad
\bar{\mathcal R}_{\mathrm{usr}}^{(M)}(\tilde Q)
=
\mathbb E_{c\sim\tilde Q}
\widehat{\mathcal R}_{\mathrm{usr}}^{(M)}(c).
\]
By convexity of \(F_{\mathrm{Env}}^{\mathrm{usr}}\), using Jensen's inequality we find 
\[
F_{\mathrm{Env}}^{\mathrm{usr}}
\left(
R_{\mathrm{usr}}(\tilde Q),
\bar{\mathcal R}_{\mathrm{usr}}^{(M)}(\tilde Q)
\right)
\le
\mathbb E_{c\sim\tilde Q}
F_{\mathrm{Env}}^{\mathrm{usr}}
\left(
\mathcal R_{\mathrm{usr}}(c),
\widehat{\mathcal R}_{\mathrm{usr}}^{(M)}(c)
\right).
\]
Applying Lemma~\ref{donsker-v} again with \(Q=\tilde Q\), \(P=P_U\), and
\[
f(c)
=
M F_{\mathrm{Env}}^{\mathrm{usr}}
\left(
\mathcal R_{\mathrm{usr}}(c),
\widehat{\mathcal R}_{\mathrm{usr}}^{(M)}(c)
\right).
\]
Then
\[
\mathbb E_{c\sim\tilde Q}
F_{\mathrm{Env}}^{\mathrm{usr}}
\left(
\mathcal R_{\mathrm{usr}}(c),
\widehat{\mathcal R}_{\mathrm{usr}}^{(M)}(c)
\right)
\le
\frac1M
\left[
D_{\mathrm{KL}}(\tilde Q\Vert P_U)
+
\log Z_{\mathrm{Env}}
\right],
\]
where
\[
Z_{\mathrm{Env}}
:=
\mathbb E_{c\sim P_U}
\exp
\left(
M F_{\mathrm{Env}}^{\mathrm{usr}}
\left(
\mathcal R_{\mathrm{usr}}(c),
\widehat{\mathcal R}_{\mathrm{usr}}^{(M)}(c)
\right)
\right).
\]
By Lemma~\ref{lem:env_exp_from_subgaussian_usr},
\[
\mathbb E_{\theta_1,\ldots,\theta_M}[Z_{\mathrm{Env}}]\le C_{\mathrm{Env}}.
\]
Hence, by Markov's inequality, with probability at least \(1-\delta/2\),
\[
Z_{\mathrm{Env}}
\le
\frac{2C_{\mathrm{Env}}}{\delta}.
\]

Consequently, with probability at least \(1-\delta/2\),
\[
F_{\mathrm{Env}}^{\mathrm{usr}}
\left(
R_{\mathrm{usr}}(\tilde Q),
\bar{\mathcal R}_{\mathrm{usr}}^{(M)}(\tilde Q)
\right)
\le
\frac{
D_{\rm KL}(\tilde Q\Vert P_U)
+
\log\frac{2C_{\mathrm{Env}}}{\delta}
}{M}.
\]
Using the calibration property of \(F_{\mathrm{Env}}^{\mathrm{usr}}\), we obtain
\[
R_{\mathrm{usr}}(\tilde Q)
\le
\bar{\mathcal R}_{\mathrm{usr}}^{(M)}(\tilde Q)
+
G_{\mathrm{Env}}^{\mathrm{usr}}(B_e),
\]
where 
\[
B_e
:=
\frac{
D_{\rm KL}(\tilde Q\Vert P_U)
+
\log\frac{2C_{\mathrm{Env}}}{\delta}
}{M}.
\]
\paragraph{Combining the two levels.}
Intersecting the high-probability events from Step 1 and Step 2 gives an event
of probability at least \(1-\delta\). On this event,
\[
R_{\mathrm{usr}}(\tilde Q)
\le
\bar{\mathcal R}_{\mathrm{usr}}^{(M)}(\tilde Q)
+
F_{\mathrm{Env}}^{\mathrm{usr}}
\left(
\frac{
D_{\mathrm{KL}}(\tilde Q\Vert P_U)
+
\log\frac{2}{\delta}
}{M}
\right),
\]
and
\[
\bar{\mathcal R}_{\mathrm{usr}}^{(M)}(\tilde Q)
\le
\widehat{\mathcal R}_{\mathrm{usr}}^{(M,N)}(\tilde Q)
+
F_{\mathrm{Task}}^{\mathrm{usr}}
\left(
\frac{
D_{\mathrm{KL}}(\tilde Q\Vert P_U)
+
\log\frac{2M}{\delta}
}{N}
\right).
\]
Substituting the second inequality into the first gives
\begin{align}
R_{\mathrm{usr}}(\tilde Q)
\le
\widehat{\mathcal R}_{\mathrm{usr}}^{(M,N)}(\tilde Q)
+
F_{\mathrm{Task}}^{\mathrm{usr}}
\left(
\frac{
D_{\mathrm{KL}}(\tilde Q\Vert P_U)
+
\log\frac{2M}{\delta}
}{N}
\right)
+
F_{\mathrm{Env}}^{\mathrm{usr}}
\left(
\frac{
D_{\mathrm{KL}}(\tilde Q\Vert P_U)
+
\log\frac{2}{\delta}
}{M}
\right),\label{thrm5_nonspecific_bound}
\end{align}
as required.

Moreover, to obtain an explicit bound we first note that
by setting \( F_{\mathrm{Env}}^{\mathrm{usr}}(a,b)=\lambda_e(a-b)^2,\) implies that \(k^{\rm usr}_e=1\) and \(G_{\mathrm{Env}}^{\mathrm{usr}}(u)=\sqrt{\frac{u}{\lambda_e}}\). Using this and after substituting the result of \eqref{thrm4_specific_bound} into \eqref{thrm5_nonspecific_bound},  we then deduce that
\begin{align}
 R_{\mathrm{usr}}(\tilde Q)
\le
\widehat{\mathcal R}_{\mathrm{usr}}^{(M,N)}(\tilde Q)
 +\sqrt{\frac{
D_{\mathrm{KL}}(\tilde Q\Vert P_U)
+
\log\frac{2M}{\delta}
}{N}} +\lambda_e ^{-\frac{1}{2}}\sqrt{
\frac{
D_{\rm KL}(\tilde Q\Vert P_U)
+
\log\frac{2C_{\mathrm{Env}}}{\delta}
}{
M
}
}.
\end{align}
\end{proof}

\end{document}